\documentclass[twoside]{article}

\usepackage[accepted]{aistats2023}

\input{macros}

\begin{document}

\twocolumn[

\aistatstitle{Manifold Restricted Interventional Shapley Values}

\aistatsauthor{ Muhammad Faaiz Taufiq\footnotemark[1] \And Patrick Bl\"{o}baum \And Lenon Minorics }

\runningauthor{Muhammad Faaiz Taufiq, Patrick Bl\"{o}baum, Lenon Minorics}

\aistatsaddress{
University of Oxford 
\And  
Amazon Research 
\And 
Amazon Research } ]

\footnotetext[1]{Work done during internship at Amazon Research.  Correspondence to: \texttt{muhammad.taufiq@stats.ox.ac.uk}.}

\begin{abstract}
Shapley values are model-agnostic methods for explaining model predictions. Many commonly used methods of computing Shapley values, known as \emph{off-manifold methods}, rely on model evaluations on out-of-distribution input samples. Consequently, explanations obtained are sensitive to model behaviour outside the data distribution, which may be irrelevant for all practical purposes. While \emph{on-manifold methods} have been proposed which do not suffer from this problem, we show that such methods are overly dependent on the input data distribution, and therefore result in unintuitive and misleading explanations. To circumvent these problems, we propose \emph{ManifoldShap}, which respects the model's domain of validity by restricting model evaluations to the data manifold. We show, theoretically and empirically, that ManifoldShap is robust to off-manifold perturbations of the model and leads to more accurate and intuitive explanations than existing state-of-the-art Shapley methods.
\end{abstract}

\section{INTRODUCTION}
Explaining model predictions is highly desirable for reliable applications of machine learning. This is especially important in risk-sensitive settings like medicine and credit scoring \citep{medicineshap, medicinexai, medicine-interp, credit-scoring} where an incorrect model prediction could prove very costly. Explainability is becoming increasingly relevant because of regulations like the General Data Protection Regulation \citep{gdpr}, which may require being able to explain model predictions before deploying a model in the real world. 
This is less of a challenge in models like linear models and decision trees, which tend to be easier to interpret. 
However, the same is not true for more complex models like Neural Networks, where explaining predictions may not be straightforward \citep{why-should-i-trust}.

Explainable AI is an area of machine learning which aims to provide methodologies for interpreting model predictions. Various different techniques of explaining models have been proposed, with each approach satisfying different properties \citep{xai-review}. In this paper, we focus on Shapley values \citep{shap1, shap2, shap-og}, a popular approach for quantifying feature relevance, which is model-agnostic, i.e., is independent of model implementation. Additionally, this is a local explanation method, i.e., it can be used to explain individual model predictions. Shapley values are based on ideas from cooperative game theory \citep{game-theory} and come with various desirable theoretical properties \citep{kernelshap} which make it a very attractive method in practice.

At a high-level, Shapley values treat features as `players' in a game, where the total payout is the model prediction at a given point. To quantify the feature importance, this method distributes the total payout among each player in a `fair' manner using a \emph{value} function. Different types of Shapley value functions have been proposed which differ in the way they distribute payout among players \citep{kernelshap, expondatamanifold}. These can be broadly divided into two categories: (i) \emph{on-manifold} value functions, which only depend on the model behaviour on the input data distribution, and (ii) \emph{off-manifold} value functions which also depend on the model behaviour outside the input data distribution. 

Off-manifold Shapley values are not robust to changes in model behaviour outside the data distribution. This means that the explanations obtained using these methods may be highly influenced if the model behaviour outside the data distribution changes, even if it remains fixed on the data distribution \citep{expondatamanifold, foolingshap, on-manifold-off-manifold}. 
Such changes to the model can change the Shapley values drastically, resulting in misleading explanations, and can even be used to hide model biases.
On the other hand, while the on-manifold Shapley values are robust to such model perturbations, the explanations obtained are highly sensitive to changes in the feature distribution. Additionally, these methods do not capture the \emph{causal} contribution of features as they attribute importance based on feature correlations. 
For example, we show that on-manifold Shapley values can be `fooled' into attributing similar importance to two positively correlated features, even if the model depends on only one of them.

In this paper, we bridge this gap between \emph{on-manifold} and \emph{off-manifold} Shapley values by proposing ManifoldShap (illustrated in Figure \ref{fig:manifoldshap}), a Shapley value function, which remains robust to changes in model behaviour outside the data distribution, while estimating the \emph{causal} contribution of features. We show that ManifoldShap is significantly less sensitive to changes in the feature distribution than other on-manifold value functions. We extend the formal notion of robustness in \citet{on-manifold-off-manifold} by providing an alternative definition which may be more desirable in many cases. We additionally show that our proposed method satisfies both notions of robustness, while other methods do not. Moreover, ManifoldShap satisfies a number of other desirable properties which we verify theoretically and empirically on real-world datasets.

\begin{figure}[t]
    \centering
    \includegraphics[height=1.5in]{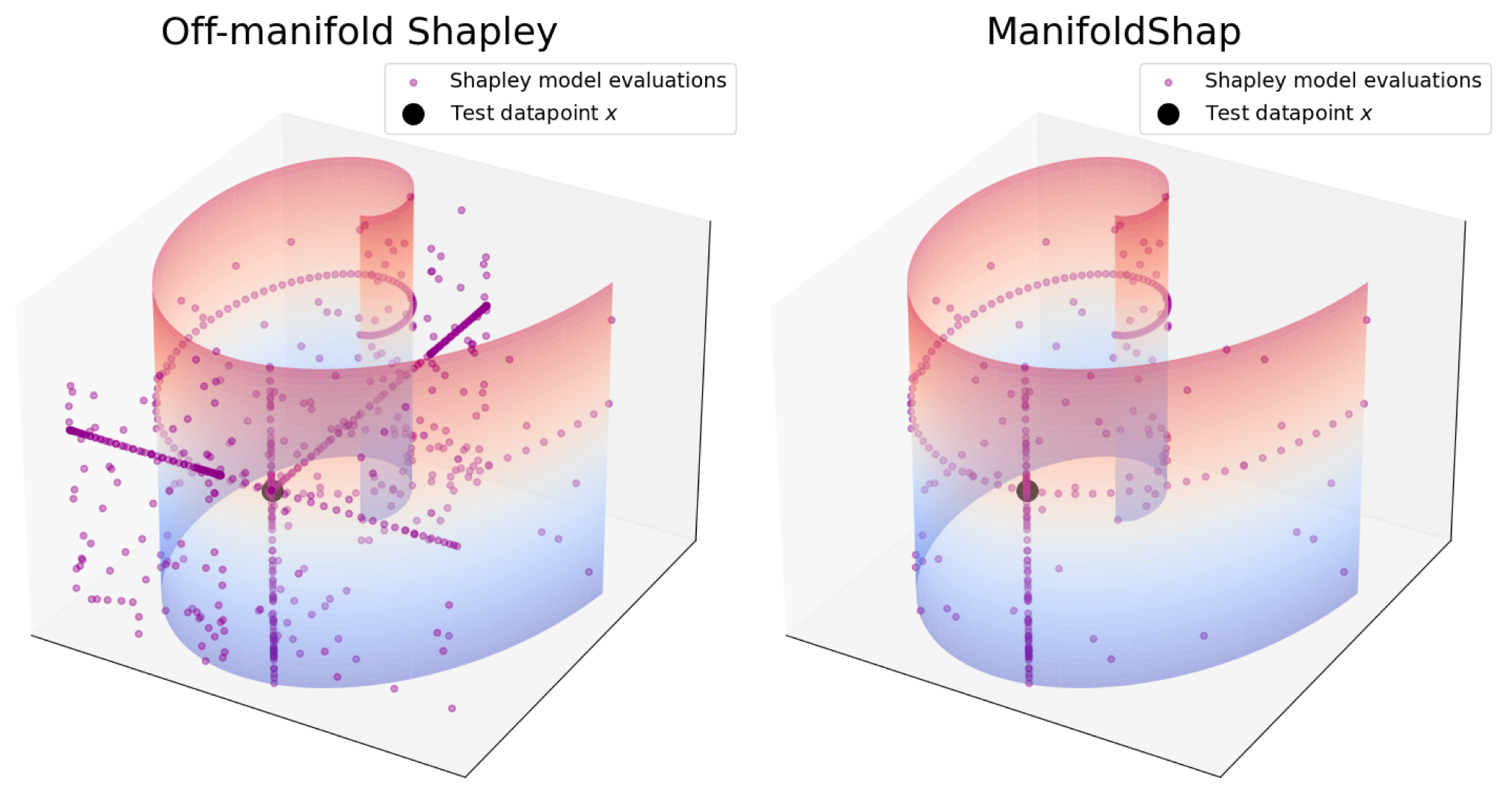}
    \caption{The datapoints at which model is evaluated when computing Shapley values for test point $\x$, along with the data manifold. Off-manifold methods evaluate the model outside the data manifold whereas our proposal, ManifoldShap, restricts model evaluations to the data manifold.}
    \label{fig:manifoldshap}
\vspace{-0.2in}
\end{figure}

	\section{SHAPLEY VALUES}
	In this section, we will introduce Shapley values for model explainability.
	For any given model $f:\mathcal{X} \rightarrow \mathcal{Y}$, our goal is to obtain localised model explanations at a given point $\x \in \mathcal{X}$. We assume that $\mathcal{X} \subseteq \mathbb{R}^d$ and $\mathcal{Y} \subseteq \mathbb{R}$. 
	
	Shapley values \citep{shap1, shap2, shap-og} provide a natural tool for obtaining such explanations. 
	For a specific input $\x$, Shapley values define a way of distributing the difference between $f(\x)$ and a baseline, which we denote as $b_0$, among the $d$ input features. 
	This can naturally be interpreted as the contribution of each feature towards the difference $f(\x) - b_0$, and is commonly referred to as feature attributions. 
	One possible choice of baseline explored in the literature is the model evaluated at an auxiliary input $\x'$, i.e., $b_0 = f(\x')$. Alternatively, many methods use the average model output $\E[f(\X)]$ as the baseline, i.e., $b_0 = \E[f(\X)]$. This can be used to explain \emph{why} the output at a point $\x$ deviates from the average output. The average output provides a more intuitive and interpretable baseline compared to the choice of an auxiliary input $\x'$, which can be arbitrary. In this work, we therefore restrict our attention to the latter category.

	As an example, consider a model which predicts an individual's salary, with input features corresponding to individual's information. If feature $i \in [d]$ represents the age of the individual, the attribution for feature $i$, which we will denote as $\atr{i}{f}$, tells us the contribution of individual's age to the salary prediction for $\x$, relative to the average salary prediction, i.e., $f(\x) - \E[f(\X)]$.
	To compute the contribution for feature $i$ at $\x$, Shapley values consider a value function $v: 2^{[d]} \rightarrow \mathbb{R}$ where $v$ may implicitly depend on $\x$. Given a subset $S\subseteq[d]\setminus\{i\}$, we can intuitively interpret the difference $v(S\cup \{i\}) - v(S)$ as the contribution of feature $i$ w.r.t. the set $S$. 
	Next, the Shapley values for feature $i$ is defined as a weighted sum over all possible subsets $S$:
	\[
	\atr{i}{f} \coloneqq \sum_{S \subseteq [d] \setminus \{i\}} \frac{|S|!(d-|S|-1)!}{d!} (v(S\cup \{i\}) - v(S)).
	\]
	The quantity $\phi_i$ can be intuitively considered as the average contribution of feature $i$ to the prediction at $\x$.
\begin{comment}
	One such family of value functions consider feature relevance relative to an auxiliary input $\x'$, and attributes the difference between $f(\x)$ and $b_0=f(\x')$ to individual features, i.e., $\sum_i \phi_i = f(\x)- f(\x')$. 
%
	A second class of value functions consider the baseline $b_0 = \E[f(\X)]$, and attribute the difference $f(\x) - \E[f(\X)]$ among individual features.
	This can be used to explain \emph{why} the output at a point $\x$ deviates from the average output. The average output provides a more natural baseline compared to the choice of an auxiliary input $\x'$, which can be arbitrary. In this work, we therefore restrict our attention to the latter category.
\end{comment}
	In order for the explanations obtained to be interpretable and intuitive, the value function $v$ must be chosen such that it satisfies a number of desirable properties.
	We present some of the most important such properties here\done{comment about symmetry}:
	\begin{enumerate}
	    \item \textit{Sensitivity:} If $f$ does not depend on $x_i$, then $v(S\cup \{i\}) = v(S)$, and hence $\atr{i}{\x}=0$.
	    \item \textit{Symmetry:} If $f$ is symmetric in components $i$ and $j$ and $x_i = x_j$, then $v(S\cup \{i\}) = v(S\cup \{j\})$ and hence $\phi_i = \phi_j$.
	    \item \textit{Efficiency:} If $\atr{i}{\x}$ denotes the attribution of feature $i$ to $f(\x) - \E[f(\X)]$, then $v([d])-v(\emptyset) = f(\x) - \E[f(\X)]$ and hence,
	    $
	    \sum_i \phi_i = f(\x) - \E[f(\X)].
	    $
	\end{enumerate}
	 Next, we present various commonly used value functions, which can be classified into \emph{off-manifold} and \emph{on-manifold} value functions.
	
	\subsection{Off-Manifold Value Functions}
	This class of value functions does not restrict function evaluations to the data distribution, and consequently, computing Shapley values involves evaluating the model on out-of-distribution inputs, where the model has not been trained (see Figure \ref{fig:manifoldshap}). The most commonly used off-manifold value function is Marginal Shapley (MS) (also called RBShap \citep{kernelshap}):
	\paragraph{Marginal Shapley (MS).}
	\[
	v^{\textup{MS}}_{\x, f}(S) \coloneqq \E[f(\textbf{x}_S, \textbf{X}_{\bar{S}})].
	\]
	Specifically, Marginal Shapley takes the expectation of $f(\textbf{x}_s, \textbf{X}_{\bar{S}})$ over the marginal density of $\textbf{X}_{\bar{S}}$.
	
	In addition to this, there has been some recent work proposing a causal perspective when computing Shapley values \citep{lshap, causalshap, jung2022}. Specifically, these works observe that manually fixing the values of features $\X_S$ to $\x_S$ when computing Shapley values, corresponds to \emph{intervening} on the feature values. In Pearl's do calculus \citep{pearl, pearl2012the}, this is expressed as $do(\X_S = \x_S)$. This leads to the definition of Interventional Shapley (IS) value functions:
	\paragraph{Interventional Shapley (IS).}
	\begin{align}
		v^{\textup{IS}}_{\x, f}(S) \coloneqq \E[f(\X) \mid do(\X_S = \x_S)]. \label{causalshap}
	\end{align}

	A detailed discussion of how Interventional Shapley differs from other \textit{non-causal} value functions has been deferred to Section \ref{subsec:limitations-on-man}.
	How to compute $v^{\textup{IS}}_{\x, f}(S)$ depends on the causal structure of the features. \citet{lshap} only consider the causal relations between the function inputs and outputs, rather than between the real-world features and the true output $Y$. This corresponds to the set-up in Figure \ref{fig:dag}, where the true feature values $\tilde{X}_i$ are formally distinguished from the features $X_i$ input into the function, $f$, with $X_i$ being a direct causal descendant of $\tilde{X}_i$ and no interactions between $X_i$. In this set-up, intervening on $\X_S$ yields the following interventional distribution:
	\begin{align*}
		p(\X_{\bar{S}} \mid do(\X_S = \x_S) ) = p(\X_{\bar{S}}).
	\end{align*}
	In this case, the value function, $v^{\textup{IS}}_{\x, f}(S)$ can straightforwardly be computed as 
	\begin{align*}
		v^{\textup{IS}}_{\x, f}(S) \hspace{-0.1cm}=\hspace{-0.1cm} \E[f(\X) \mid do(\X_S = \x_S)] \hspace{-0.1cm} = \hspace{-0.1cm} \E_{\X_{\bar{S}} \sim p(\X_{\bar{S}})}[f(\x_S, \X_{\bar{S}})].
	\end{align*}
 \begin{figure}[ht]
	    \centering
	    \includegraphics[height=1.2in]{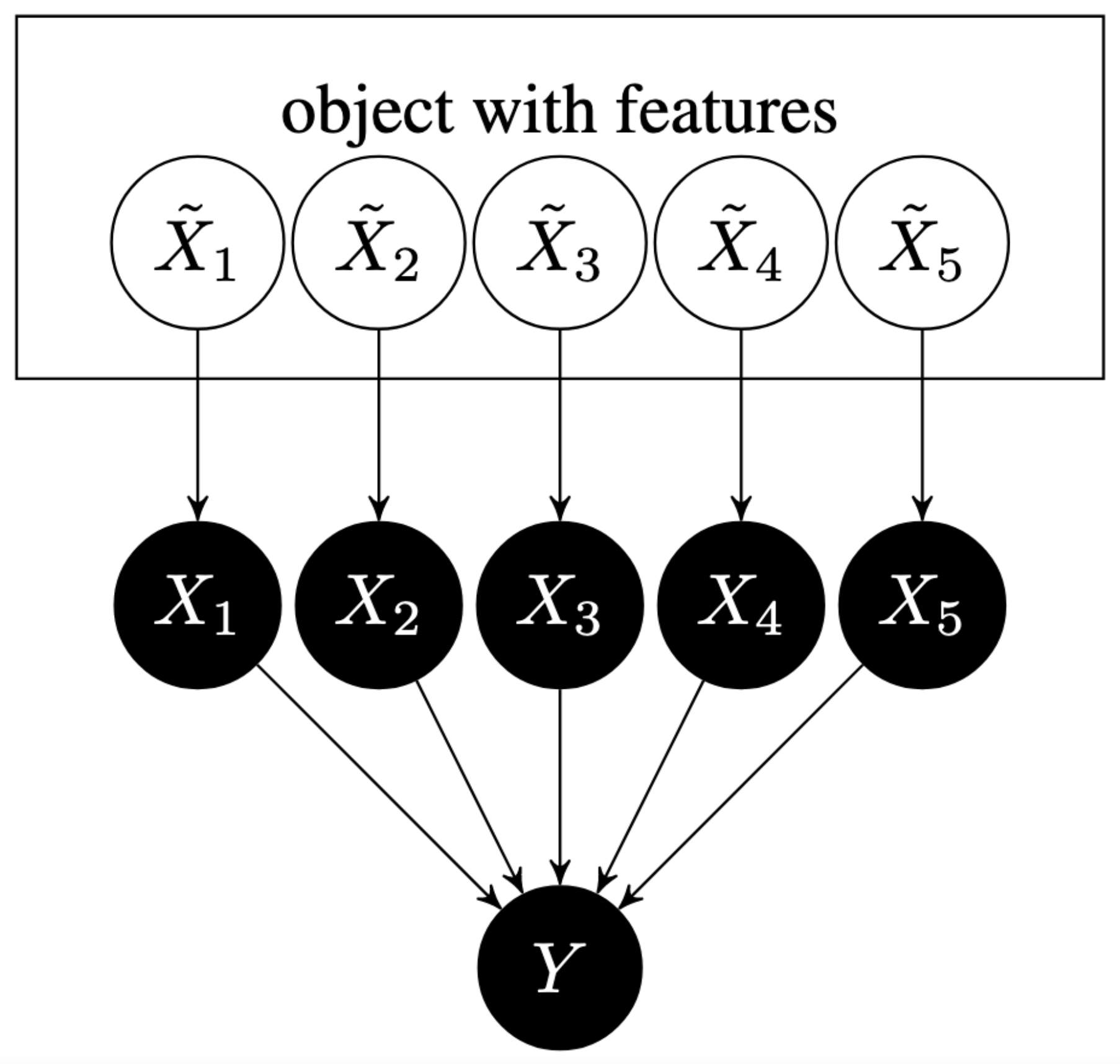}
	    \caption{Causal structure considered in \citet{lshap}. The true features are $\tilde{X}_i$ while features input into the model are $X_i$.}
	    \label{fig:dag}
	\end{figure}
 
	This is equivalent to Marginal Shapley. Therefore, Marginal Shapley can be considered a special case of Interventional Shapley. In contrast, \citet{causalshap} seeks to estimate the causal contributions of the real-world features towards the true output $Y$, and therefore,
	does not distinguish between the true features and the features input into the model. The resulting IS value function also takes into account the causal relations among the true features themselves.
	
	\subsection{On-Manifold Value Functions}
	These value functions only rely on function values in data distribution when computing Shapley values. As a result, any changes in the function outside data distribution does not change the explanations obtained. One of the first on-manifold value functions proposed was Conditional Expectation Shapley (CES) \citep{kernelshap}: 
	\paragraph{Conditional Expectation Shapley (CES).}
	\[
	v_{\textbf{x}, f}^{\textup{CES}}(S) \coloneqq \E[f(\textbf{X}) \mid \textbf{X}_S = \textbf{x}_S ].
	\]
	Unlike Marginal Shapley, CES takes the expectation of $f(\x_s, \textbf{X}_{\bar{S}})$ over the conditional density of $\textbf{X}_{\bar{S}}$ given $\textbf{X}_S = \textbf{x}_S$ (and not the marginal density of $\textbf{X}_{\bar{S}}$). This has undesired implications for the obtained Shapley values, which we discuss in detail in Section \ref{subsec:limitations-on-man}. 
	
	Apart from this, recently \citet{on-manifold-off-manifold} proposed Joint Baseline Shapley (JBShap), a value function which aims to make Shapley values robust to model changes in regions of low data-density. This value function explicitly takes the density $p(\x)$ into consideration when calculating explanations:
	
    \paragraph{Joint Baseline Shapley (JBShap).}
    \[
    v_{\textbf{x}, f, p}^{\textup{J}}(S) \coloneqq f(\textbf{x}_S, \textbf{x}'_{\bar{S}})p(\x_S, \x'_{\bar{S}}),
    \]
    where $\x'$ is an auxiliary baseline. The authors also propose an extension of JBShap, called \emph{Random Joint Baseline Shapley} (RJBShap) where the value function averages over all possible baseline values:
    \paragraph{Random Joint Baseline Shapley (RJBShap).}
    \[
    v_{\textbf{x}, f, p}^{\textup{RJ}}(S) \coloneqq \E_{p_b(\X_{\bar{S}})}[f(\textbf{x}_S, \textbf{X}_{\bar{S}})p(\x_S, \X_{\bar{S}})].
    \]
    Here, $p_b(\X_{\bar{S}})$ is some prior distribution over features $\x'_{\bar{S}}$. A natural choice of prior is the marginal density $p(\X_{\bar{S}})$, which we use to compute RJBShap later.
    
    Having listed the most relevant on and off manifold value functions, we discuss their limitations in the following sections. This will motivate our proposal of an alternative value function, which aims to circumvent these limitations.
    
    \subsection{Limitations of off-manifold value functions}
    As \citet{foolingshap, expondatamanifold} point out, dependence of Shapley explanations on off-manifold behaviour of the model can be problematic. For example, computing Interventional Shapley at $\x$ requires evaluating the model at points $(\x_S, \X_{\bar{S}})$ for $S\subseteq[d]$ where $\X_{\bar{S}} \sim p(\X_{\bar{S}} \mid do(\X_S = \x_S))$. Such points may lie outside the distribution of training data, where the model was not trained. 
    Consider a model which is identical to the ground truth function on the data distribution. The train/test errors of the model will be 0, suggesting that it captures the ground truth function perfectly. However, if the model differs from the ground truth outside the data distribution, the model's Shapley values may be drastically different from the ground truth Shapley values, resulting in highly misleading explanations.

    This limitation of off-manifold Shapley values can be exploited to `fool' Shapley values into hiding model biases. In \citet{foolingshap}, the authors consider models which are highly biased on the data manifold (i.e., solely rely on sensitive features, like racial background, for predictions). They show that these models can be perturbed outside the data manifold in such a way that the resulting Shapley values give no attribution to the sensitive features, despite the models relying solely on these sensitive features on the data manifold. Therefore, off-manifold Shapley values are highly vulnerable to off-manifold manipulations.
    
    \subsection{Limitations of on-manifold value functions}\label{subsec:limitations-on-man}
    While the on-manifold value functions do not consider model behaviour outside data distribution, the existing methods can lead to unintuitive or misleading Shapley explanations as they do not consider the \textit{causal} contributions of features, and are highly sensitive to feature correlations. Specifically, as \citet{lshap} point out, when computing feature contributions at $\x$, the value function for a subset $S$, $v(S)$, must capture the effect of fixing the feature values $\X_S$ to $\x_S$. This is \emph{not} given by $\E[f(\X) \mid \X_S = \x_S]$ as in CES, because observing $\X_S = \x_S$ also changes the distribution of $\X_{\bar{S}}$. Instead, the impact of setting $\X_S$ to $\x_S$ is captured by $\E[f(\X) \mid do(\X_S = \x_S)]$, which in general is different from conditional expectation. Therefore, Interventional Shapley is inherently proposed to capture the \emph{causal} effect of fixing feature values. 
    
    Since CES considers the conditional expectation $\E[f(\X) \mid \X_S = \x_S]$ when computing Shapley values, the resulting Shapley values are highly influenced by feature correlations. As a result, two highly correlated features may receive similar feature attributions even if the model under consideration depends on only one of them. We make this concrete with an example in Appendix \ref{sec:int-vs-ces}\done{add RJBShap to this}. We also demonstrate empirically in Section \ref{sec:exps} and Appendix \ref{sec:exps-app} that CES can be highly sensitive to the feature correlations, and consequently can lead to wrong explanations. Additionally, computing CES is computationally challenging when the feature-space is continuous. While \citet{expondatamanifold} propose training a surrogate model $g$ with masked inputs to estimate the conditional expectation (see Appendix \ref{subsec:CES-comp}),  training $g$ is even more difficult than training the model $f$.
	
	Aside from this, the JBShap and RJBShap value functions proposed by \citet{on-manifold-off-manifold}, explain the feature contributions for the function $\tilde{f}_p(\x) \coloneqq f(\x)p(\x)$, rather than $f(\x)$ itself. Specifically, RJBShap explain the contribution of individual features towards the difference $\tilde{f}_p(\x) - \E_{p_{b}(\X)}[\tilde{f}_p(\X)]$.
	This means that the resulting Shapley values therefore do not explain the underlying function $f$ itself.
 We make this more concrete with an example with $\Xspace \subseteq \mathbb{R}^2$:
 \begin{align}
     \X \sim \mathcal{N}(\textbf{0}, I_2), \quad f(\x) = \exp{\left(x^2_1/2\right)}. \label{eq:rjbshap-example}
 \end{align}
 For this example, $\tilde{f}_p(\x)$ only depends on $x_2$ and consequently, the RJBShap values for feature 1, $\phi_1 = 0$, for all $\x \in \mathcal{X}$, even though the function $f(\x)$ \emph{only} depends on $x_1$. RJBShap can therefore lead to \emph{highly} misleading explanations. We confirm this empirically in Appendix \ref{subsec:rjbshap}.
	Additionally, the notion of off-manifold robustness satisfied by JBShap and RJBShap value functions can be restrictive. We expand upon this in Section \ref{subsec-robustness}, where we propose an alternative definition of robustness which is less restrictive, and is not satisfied by JBShap and RJBShap.
	
	\done{Retrieve the following commented out bit}
	\done{More stress on the computational aspect}

 \section{MANIFOLD RESTRICTED SHAPLEY VALUES}
In this paper, we argue that a model must be mainly characterised by it's behaviour on the data manifold. While \emph{intervening} on features provides the correct notion of fixing features, we must restrict our attention to the data manifold when estimating Shapley values. This allows us to avoid the issues of non-identifiability outside the data manifold, thereby making the Shapley estimates robust against adversarial attacks as in \citet{foolingshap}.  
	In order to estimate Shapley values which are robust to off-manifold manipulations, we must restrict the function evaluation to the data manifold. Before we proceed, we introduce our value function in terms of general sets $\Z \subseteq \mathcal{X}$.
	
	\begin{definition}[ManifoldShap]
	    Let $\Z \subseteq \mathcal{X}$ be an open set with $\x \in \Z$, and $\p(\X\in\Z \mid do(\X_S = \x_S)) > 0$ for $S\subseteq [d]$. Then, we define the ManifoldShap on $\Z$ as:
	    \begin{align}
	        \valgeneric{f}{\Z}(S) \coloneqq \valfunset{S}{f}{\Z}. \label{val_fun}
	    \end{align}
	\end{definition}
	\done{In Appendix \ref{proofs} we prove that under mild regularity assumptions, the value function \eqref{val_fun} is well-defined. IMPORTANT: Prove this.}
	\textbf{Remark.} The notation $\E[\cdot \mid do(\X_S =\x_S), \X\in \Z]$ denotes the expectation w.r.t. the density $p_{\Z, \x_S}(\cdot)$ where
	   \begin{align}
	       p_{\Z, \x_S}(\y) \coloneqq \frac{p(\y\mid do(\X_S =\x_S))\ind(\y\in\Z)}{\p(\X\in\Z \mid do(\X_S = \x_S))}. \label{man-density}
	   \end{align}

    The condition $\p(\X\in\Z \mid do(\X_S = \x_S)) > 0$ ensures that $p_{\Z, \x_S}(\x)$ (and hence $\valgeneric{f}{\Z}(S)$) is well-defined.
	By conditioning on the event $\X \in \Z$, the ManifoldShap value function restricts the function evaluations to the set $\Z$. In practice, $\Z$ can be 
	chosen to be the data manifold, or any other region of interest, where model behaviour is relevant to explanations sought. In this way, ManifoldShap will disregard the model behaviour outside the region of interest when computing Shapley values. A detailed discussion of how to choose the sets $\Z$ is deferred to the next section.
	
	Our formulation of \emph{ManifoldShap} is general as it is does not assume a specific causal structure on the features. In our methodology, we assume that the expectation $\E[f(\textbf{X}) \mid do(\textbf{X}_S = \textbf{x}_S)]$ can be computed using observational data. This is a standard assumption needed to compute Interventional Shapley, and holds true under the causal structure in Figure \ref{fig:dag}. Under this assumption, we can compute the value function using the following result. 
	\begin{lemma}\label{manifoldShap}
		The value function $\valgeneric{f}{\Z}$ can be written as,
		\begin{align*}
			\valgeneric{f}{\Z}(S) = \frac{\E[f(\textbf{X}) \ind(\X \in \Z) \mid do(\textbf{X}_S = \textbf{x}_S)]}{\p(\X \in \Z \mid do(\textbf{X}_S = \textbf{x}_S))}
		\end{align*}
	\end{lemma}

    In practice, all we need is a manifold classifier, trained to estimate the value of the indicator, i.e. $\hat{g}(\textbf{x}) \approx \ind(\textbf{x} \in \Z)$.  The value function \eqref{val_fun} can then be estimated using:
	\begin{align}
		\valgeneric{f}{\Z}(S) %
		 &\approx \frac{\E[f(\X) \hat{g}(\X) \mid do(\textbf{X}_S = \textbf{x}_S)]}{\E[\hat{g}(\X) \mid do(\textbf{X}_S = \textbf{x}_S)]}. \label{val_fun_approx}
	\end{align}
	We also provide alternative methodologies of estimating ManifoldShap using rejection sampling and regression techniques in Appendix \ref{subsec:manshap-alternative-methods}.

	\paragraph{Choosing the sets $\Z$.}
	Next, we discuss general purpose methodologies of choosing sets $\Z$ which can serve as practical estimation of the data manifold in most cases. One can obtain $\Z$ by training an out-of-distribution classifier directly. \citet{foolingshap} do so by perturbing each datapoint on randomly chosen features, and subsequently using these to train the classifier. In general, users may wish to choose different regions of interest $\Z$ on an ad hoc basis when computing Shapley values. In what follows, we outline a few specific choices of $\Z$, each of which satisfy different notions of robustness to off-manifold manipulations. We discuss this in greater length in Section \ref{subsec-robustness}.
	\begin{definition}[Density manifold]\label{den-manifold}
		Given an $\epsilon > 0$, we define the \emph{$\epsilon$-density manifold} (\emph{$\epsilon$-DM}) of the data distribution, denoted as $\man$, as:
		$
		\man \coloneqq \{\textbf{x}\in \mathbb{R}^d : p(\textbf{x}) > \epsilon \}.
		$
		Here, $p(\textbf{x})$ denotes the joint density of the data.
	\end{definition}
	The $\epsilon$-DM includes all regions of high density in the set.
	Using $\Z = \man$ in our value function therefore restricts function evaluations to regions of high density.
	An alternative way to choose $\Z$ is via the probability mass captured by $\Z$, i.e., for a given level $\alpha$, we may pick sets $\Z=\alphaman$ such that $\p(\X\in \alphaman)\geq \alpha$. One such set can be defined as:
	\begin{definition}[Mass manifold]\label{mass-manifold}
	    Given an $\alpha > 0$, we define the \emph{$\alpha$-mass manifold} ($\alpha$-MM) of the data distribution, denoted as $\alphaman$, as $\alphaman \coloneqq \mathcal{D}_{\epsilon^{(\alpha)}}$, where 
	    $
        \epsilon^{(\alpha)} \coloneqq \sup \{\epsilon \geq 0: \p(\X\in\man)\geq \alpha \}.
        $
	\end{definition}
	We show in Proposition \ref{optimality} (Appendix \ref{proofs}) that the Lebesgue measure of $\alphaman$ is smallest among the sets $\Z$ with $\p(\X \in \Z) \geq \alpha$. It should be noted that $\alphaman$ is not necessarily the unique such set. 
	One can use techniques like kernel density estimation and VAEs to approximate the manifolds described in this section (more details in Appendix \ref{subsec:computingman}). 

	\subsection{Robustness to off-manifold manipulation}\label{subsec-robustness}
	We say that a Shapley value function is robust to off-manifold manipulation, if changing the model $f$ outside the data manifold does not lead to `large' changes in its Shapley values. In this section, we formalise this idea of robustness and show that ManifoldShap satisfies this notion, while the existing value functions do not.
	First, we present the definition of robustness as used in \citet{on-manifold-off-manifold}, to formalise the notion of off-manifold manipulations.
	\begin{definition}[T-robustness \citep{on-manifold-off-manifold}]\label{trobustness}
		Given two models $f_1(\textbf{x}), f_2(\textbf{x})$ and any probability density $p(\textbf{x})$, we say that a value function, $v_{\textbf{x}, f}$, is strong T-robust if it satisfies the following condition: if $\max_{\textbf{x}} | f_1(\textbf{x}) - f_2(\textbf{x}) | p(\textbf{x}) \leq \delta$, then, $| v_{\textbf{x}, f_1}(S) - v_{\textbf{x}, f_2}(S)| \leq T \delta$ for any $S \subseteq [d]$.
	\end{definition}
	
	As per \citet{on-manifold-off-manifold},``The premise $\max_{\textbf{x}} | f_1(\textbf{x}) - f_2(\textbf{x}) | p(\textbf{x}) \leq \delta$ bounds the maximum perturbation on low density regions." Additionally, \citet{on-manifold-off-manifold} show that JBShap and RJBShap value functions satisfy strong T-robustness to off-manifold manipulation, while other value functions like MS and CES do not. Likewise, since MS is a special case of IS, the latter also does not satisfy strong T-robustness. On the other hand, ManifoldShap restricted to $\epsilon$-density manifold, $\man$, satisfies this notion of robustness. 
	\begin{proposition}\label{t-robust}
		The value function $\valgeneric{f}{\man}(S) = \valfunset{S}{f}{\man}$ is strong $T$-robust for $T = 1/\epsilon$.
	\end{proposition}
	Proposition \ref{t-robust} shows that with decreasing $\epsilon$, the robustness parameter $T$ increases and ManifoldShap gets less robust.

\paragraph{Alternative definition of Robustness.}

Definition \ref{trobustness} considers a very specific notion of model perturbation. In particular, the perturbation in model $f(\x)$ must not exceed $\delta/p(\x)$ for all $\x \in \mathbb{R}^d$ and some $\delta>0$. This does not encapsulate the case where the function perturbation remains bounded on a region of interest $\Z$, but may increase arbitrarily outside $\Z$. For example, we may have the case that the function $f(\x)$ remains fixed on a set $\Z$ with $\p(\X \in \Z) > 0.99$. Robustness of Shapley values should dictate that changing the function outside $\Z$ should not lead to arbitrarily different Shapley values. We later show that Def. \ref{trobustness} does not lead to such robustness guarantees. \done{show?}
To encapsulate this, we provide an alternative definition of robustness, which allows us to take into account model manipulation on sets with small probability mass. 
First, we define the notion of robustness on a general feature subspace $\Z'\subseteq \Xspace$:

\begin{definition}[Subspace T-robustness]\label{subspacerobustness}
    Let $\Z'\subseteq\mathcal{X}$ be such that $\p(\X\in\Z')>0$. We say that a value function $v_{\textbf{x}, f}$ is strong T-robust on subspace $\Z'$ if it satisfies the following condition: if $\sup_{\x\in\Z'} | f_1(\textbf{x}) - f_2(\textbf{x}) | \leq \delta$, then, $| v_{\textbf{x}, f_1}(S) - v_{\textbf{x}, f_2}(S)| \leq T \delta$ for any $S \subseteq [d]$.
\end{definition}

A value function satisfying strong T-robustness on $\Z$ would not result in drastically different Shapley values when the model perturbation is bounded on the set $\Z$, by some value $\delta > 0$.
The above definition allows us to directly consider robustness of value functions on sets based on probability mass, $\alphaman$. Moreover, by restricting the function evaluations to a set $\Z$, ManifoldShap is naturally set up to provide subspace T-robustness guarantee. We formalise this as follows:
\begin{proposition}\label{manshap-subspace-robustness}
    The value function $\valgeneric{f}{\Z}$ is strong T-robust on any set $\Z'$ satisfying $\Z\subseteq \Z'$ with $T = 1$.
\end{proposition}
%
%
%
\begin{comment}
In contrast, we show that value functions such as IS, CES, MS do not satisfy this notion of robustness. In addition to this, while the JBShap and RJBShap value functions proposed by \citep{on-manifold-off-manifold} satisfy strong T-robustness (Def. \ref{trobustness}), they do not satisfy subspace T-robustness either. Consequently, the Shapley explanations can change drastically if the function perturbations are unconstrained in sets of small probability mass.  We also verify this empirically in Section \ref{sec:exps}. 
\end{comment}
In contrast, we show that all other value functions under consideration do not satisfy this notion of robustness:
\begin{proposition}\label{subspace-robustness-causalshap}
    For any set $\Z'$ with $\p(\X\in \Z') < 1$, the IS value function $v^{\textup{IS}}_{\x, f}(S)$, the CES value function $v_{\textbf{x}, f}^{\textup{CES}}(S)$, and the MS value function $v_{\textbf{x}, f}^{\textup{MS}}(S)$, the JBShap value function $v_{\textbf{x}, f}^{\textup{J}}(S)$ and the RJBShap value function $v_{\textbf{x}, f}^{\textup{RJ}}(S)$ are all \emph{not} strong T-robust on subspace $\Z'$ for $|T| < \infty$.
\end{proposition}
Consider the family of value functions which \emph{drop} features in $\bar{S}$ through randomisation, i.e., $v_{f, p_S}(S) = \E_{\X \sim p_S}[f(\X)]$. We note that IS, MS, CES and ManifoldShap all fall into this family. For example, when $p_S = p(\X \mid do(\X_S = \x_S))$ we obtain IS, and when $p_S = p(\X \mid \X_S = \x_S)$ we obtain CES. We show in Appendix \ref{subsec:tv-distance} that the choice of $p_S$ in ManifoldShap (i.e. $p_{\Z, \x_S}$ in Eq. \eqref{man-density})  minimises the Total Variation distance with interventional distribution $p(\X \mid do(\X_S = \x_S))$ subject to the condition that $v_{f, p_S}(S)$ is strong T-robust on $\Z$. This ensures that ManifoldShap values provide reasonable estimation of \emph{causal} contribution of features. 

\done{Perhaps combine the two sets of robustness results}

\done{combine the proofs for JBShap and RJBShap. Formalise the meaning of `for any $\Z$'}
\done{Our methodology has an implicit dependence on the density, through the $\ind$, whereas JBshap directly depends on it.}

\done{Choice of baseline as $\x'$ is largely random?}

\subsection{Comparison with existing methods}\label{sec:comparison}
\textbf{Causal Accuracy. }Recall that, CES attributes feature importance based on feature correlations. Consequently, two highly correlated features may be attributed similar feature importance even if the model only depends on one of them, i.e., the sensitivity property is violated. However, ManifoldShap on the other-hand, seeks to estimate the \emph{causal} contribution of features towards the prediction $f(\x)$, as it uses the \emph{interventional} measure restricted to the manifold $\Z$ to drop features. The experiments in Appendix \ref{sec:exps-app} confirm this, as the ManifoldShap results are significantly less sensitive to feature correlations than CES. 

Our example in Eq. \eqref{eq:rjbshap-example} shows how the explicit dependence of RJBShap on the density can lead to extremely inaccurate Shapley explanations. In Appendix \ref{subsec:rjbshap}, we show that because of its causal nature, ManifoldShap provides significantly more accurate and intuitive explanations. 
Additionally, unlike RJBShap, ManifoldShap only depends on the density estimation via the indicator $\ind(p(\x)\geq \epsilon)$. Therefore, as we show in Appendix \ref{subsec:sensitivity-density-error}, ManifoldShap is significantly more robust to density estimation errors than RJBShap. 

Aside from this, \cite{ghalebikesabi2021on} propose Neighbourhood SHAP, a value function aimed to provide explanations for the localised behaviour of the model near the datapoint $\x$ where explanations are sought. 
While the authors empirically show the robustness of the methodology against off-manifold perturbations, they do not consider the causal perspective and therefore the main object of interest is not the causal contribution of features.

\textbf{Robustness. }As outlined in Section \ref{subsec-robustness}, ManifoldShap is robust to model changes outside the manifold and therefore is not vulnerable to adversarial attacks as in \citet{foolingshap}. 
In light of this, we argue that ManifoldShap provides a compromise between conditional and interventional Shapley values. 
It attempts to estimate causal contributions of features, while providing robustness guarantees.

\textbf{Trade-off between Accuracy and Robustness. }Restricting function evaluations to the manifold $\Z$, as in ManifoldShap, means that the resulting Shapley values are dependant on the manifold itself, and may not purely reflect the causal contribution of features. This is because these are no longer pure Interventional Shapley values. 
This results in a trade-off between robustness to off-manifold manipulation and the `causal accuracy' of the Shapley values.
ManifoldShap provides us flexibility over this trade-off, through the size of the manifold $\Z$. 
When $\Z = \man$, the size of the manifold is modulated through the $\epsilon$ parameter.
As $\epsilon \rightarrow 0$, the size of manifold increases and ManifoldShap values tend towards IS values. However, as mentioned above, it comes at the cost of reduced robustness, as the Shapley evaluations include increasing number of datapoints `far' from the training data.
On the other hand, increasing $\epsilon$ increases the robustness of Shapley values, while reducing their causal accuracy, as the resulting Shapley values discard a significant number of datapoints which lie outside $\man$.

\textbf{Computational Considerations. }Computing CES may be computationally expensive and may require different supervised or unsupervised learning techniques \citep{expondatamanifold, kernelshap, on-manifold-off-manifold}. In contrast, while ManifoldShap requires estimating a manifold classifier, estimating $\valgeneric{f}{\Z}(S)$ does not incur any computational cost over and above computing the interventional expectations. Proposition \ref{manifoldShap} illustrates this by expressing the ManifoldShap value function as a ratio of interventional expectations. This is even more straightforward when the causal structure is as in Figure \ref{fig:dag}\done{describe the setting}, and the interventional expectation is equivalent to marginal expectation.
Additionally, to avoid the exponential time complexity of computing the value function for all $S\subseteq [d]$, we propose a sampling based estimation in Appendix \ref{subsec:rejection_sampling} which makes computation of ManifoldShap feasible for high dimensional feature spaces (see Appendix \ref{subsec:feature_dims}).
\section{ROBUSTNESS IN OTHER EXPLANATION METHODS}
Shapley value is not the only \textit{off-manifold} explanation method. This problem has also been explored in other explanation methods like LIME \citep{foolingshap, improvinglime, resistingood} and gradient-based methods \citep{foolingnn, fairwashing}. For example, \citet{foolingnn} illustrates this problem in gradient-based interpretability methods for Neural Networks. The paper shows that these explanations are not stable when model is manipulated without hurting the accuracy of the model. Numerous solutions have also been proposed such as \citet{resistingood}, which addresses this problem for explanation methods like RISE, OCCLUSION and LIME by quantifying a similarity metric for perturbed data. This metric is then integrated into the explanation methods. Likewise \citet{improvinglime} proposes to make LIME robust to off-manifold manipulation, by using a GAN to sample more realistic synthetic data which are then used to generate LIME explanations. Aside from this, \citet{fairwashing} proposes an alternative robust gradient-based explanation method. However, unlike Shapley values, gradient-based methods rely on model properties (e.g., differentiability), and are not model agnostic. 

\section{EXPERIMENTAL RESULTS}\label{sec:exps}
	In this section, we conduct experiments on synthetic and real world datasets to demonstrate the utility of ManifoldShap and compare it with existing methods. Instead of training the models, we compute Shapley values for the underlying true functions directly. Additional experiments investigating the sensitivity of the different Shapley methods to changing feature correlations, manifold size and feature dimensions have been included in Appendix \ref{sec:exps-app}. The code to reproduce our experiments can be found at \href{https://github.com/amazon-science/manifold-restricted-shapley}{\color{blue}{github.com/amazon-science/manifold-restricted-shapley}}.
	
	\subsection{Synthetic data experiments}
	
	\begin{figure*}
	\centering
	\begin{subfigure}{0.5\textwidth}
	    \centering
	    \includegraphics[height=0.9in]{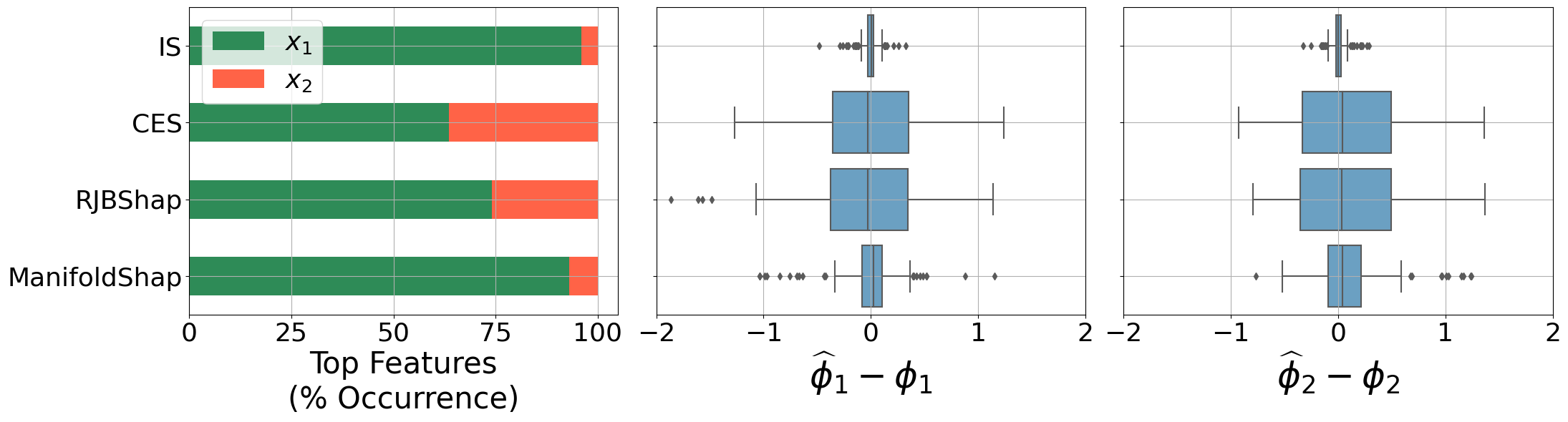}
	    \subcaption{$\delta=0$}
	    \label{fig:delta_0}
	\end{subfigure}%
	\begin{subfigure}{0.5\textwidth} 
	    \centering
	    \includegraphics[height=0.9in]{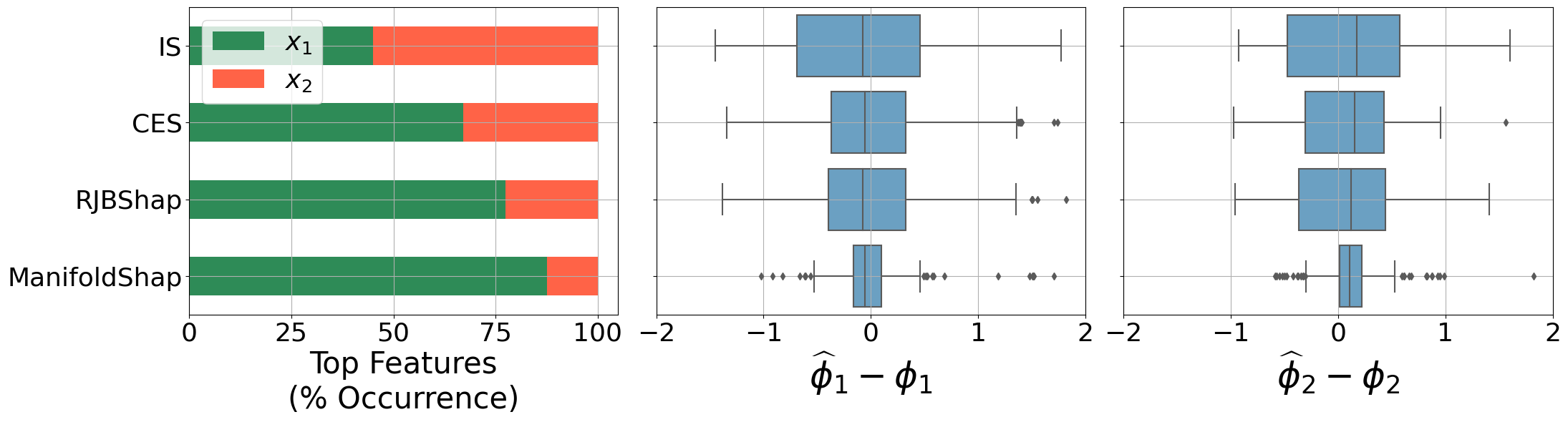}
	    \subcaption{$\delta=5$}
	    \label{fig:delta_5}
	\end{subfigure}
	\caption{Synthetic data experiments for $\delta=0, 5$. The barplots on the left of each subfigure shows the most important features for different Shapley value functions. The boxplots show the approximation errors of the Shapley values for different value functions.}\label{fig:pert-exps-dag}
    \end{figure*}

	\done{causal structure of these experiments}
	\done{talk about baselines}
	\done{how do you compute the conditional expectation}

	Here we investigate the effect of model perturbation in low density regions on Shapley values. 

\paragraph{Data generating mechanism.}
\begin{wrapfigure}{l}{0pt}
	\begin{tikzpicture}
	\tikzset{
    > = stealth,
    every node/.append style = {
        draw = black,
        shape = circle,
        inner sep = 0.5pt,
        minimum size=0.75cm
    },
    every path/.append style = {
        arrows = ->,
    }
    }
    \tikz{
        \node (x) at (0,0) {$X_2$};
        \node (y) at (2,0) {$Y$};
        \node (z) at (1,1) {$X_1$};
        \path (x) edge (y);
        \path (z) edge (x);
        \path (z) edge (y);
    }
    \end{tikzpicture}
\end{wrapfigure}
	In this experiment, $\mathcal{Y} \subseteq \mathbb{R}$ and $\mathcal{X} \subseteq \mathbb{R}^2$ follow the Causal DAG shown on the left.\done{re-add caption} %
	In specific, the Structural Causal Model (SCM) \citep{pearl} for the ground truth data generating mechanism is:
	\begin{align*}
	    X_1 &= \epsilon_1, \hspace{0.3cm} X_2 = \rho X_1 + \sqrt{1 - \rho^2} \epsilon_2, \hspace{0.3cm} Y = X_1.
	\end{align*}
	Here, $\epsilon_i \overset{\textup{i.i.d.}}{\sim}\mathcal{N}(0, 1)$ and $\rho = 0.85$ is the correlation between $X_1, X_2$.
	Next, we define the perturbed models.
	\begin{figure}[h!]
		\centering
		\begin{subfigure}[t]{0.24\textwidth}
			\centering
			\includegraphics[height=1.3in]{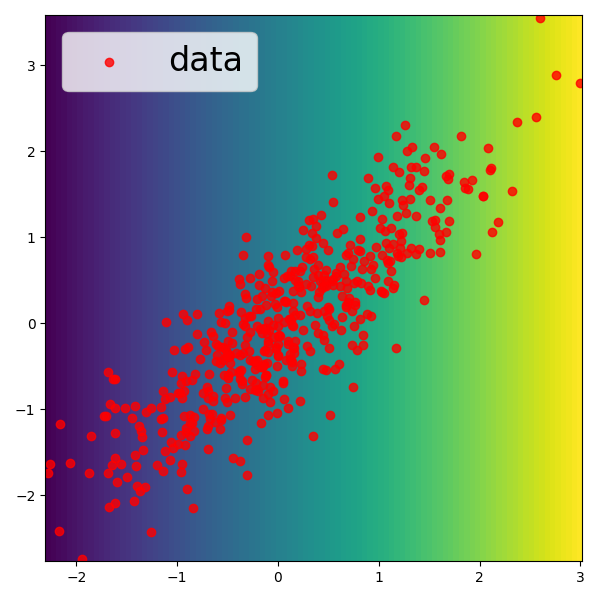}
			\subcaption{Heatmap of $g_\delta$ for $\delta = 0$.}	
		\end{subfigure}%
		\begin{subfigure}[t]{0.24\textwidth}
			\centering
			\includegraphics[height=1.3in]{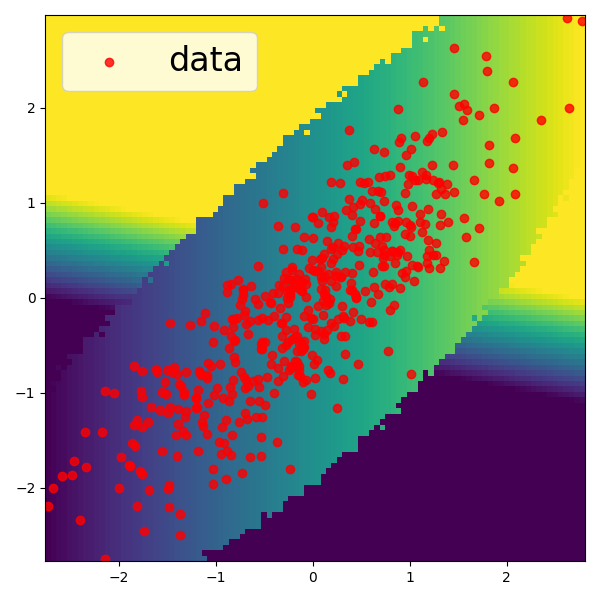}
			\subcaption{Heatmap of $g_\delta$ for $\delta = 5$.}
		\end{subfigure}
		\caption{Heatmaps for ground truth and perturbed models $g_\delta$. Each model has test mean squared error of 0.}\label{fig:heatmaps}
	\end{figure}
	\paragraph{Perturbed models.}
	We define the following family of perturbed models $g_\delta:\mathcal{X} \rightarrow \mathbb{R}$, parameterised by $\delta \in \mathbb{R}$.
	\begin{align*}
    	g_\delta(\X) \coloneqq Y + \delta X_2 \ind(\X \not \in \alphaman).
	\end{align*}
	 Here, we use VAEs to estimate $\alphaman$ (see Appendix \ref{subsec:computingman}) and choose $\alpha=1-10^{-3}$.
	By construction, the models $g_\delta$ should agree with the ground truth on the $\alpha$-manifold, i.e. $g_\delta(\X) = Y$ when $\X \in \alphaman$, but these models differ from the ground truth for $\X \not\in \alphaman$. Figure \ref{fig:heatmaps} shows the model heatmaps for $\delta = 0, 5$ along with the original data.
    It is impossible to distinguish between these models on the data manifold, as both have test mean squared error of 0.

	\paragraph{Results.}
	Recall that the ground truth model does not depend on $X_2$, so the ground truth Shapley value for feature 2 is $\phi_2 = 0$. As a result, for any prediction, feature 1 has greater absolute Shapley value than feature 2, i.e. $|\phi_1| \geq |\phi_2|$.
	We compute Shapley values for $g_\delta$ using different value functions on 500 datapoints $\{\x^{(i)}\}_{i=1}^{500}$, sampled from the SCM defined above. We compute CES using the ground truth conditional distributions of $X_i \mid X_j$ for $i\neq j$, which can be obtained analytically in this setting. Figure \ref{fig:pert-exps-dag} shows the results, with the bar plots on the left of Figures \ref{fig:delta_0} and \ref{fig:delta_5}, showing the most important features as per different value functions for $\delta=0, 5$.\done{re-add footnote}%
	
	For $\delta=0$, Figure \ref{fig:delta_0} confirms that the IS values of the ground truth model attribute greatest feature importance to feature 1 for all datapoints. This is expected as the ground truth model does not depend on $x_2$. For ManifoldShap, we observe that for 4\% of the datapoints, feature 2 is attributed greater importance. This highlights that robustness of ManifoldShap comes at the cost of reduced causal accuracy of Shapley values. Furthermore, it can be seen that CES value function attributes greatest importance to feature 2 for more than 30\% of the datapoints. This is because CES provides similar Shapley values for positively correlated features. We observe similar behaviour for RJBShap, which attributes greatest importance to feature 2 for about 20\% of datapoints.
	This happens because RJBShap provides feature contributions for $\tilde{f}_p(\x)=f(\x)p(\x)$ rather than $f(\x)$, and can therefore be misleading. \done{elaborate}
	
	When $\delta=5$, Figure \ref{fig:delta_5} shows that, for more than 50\% of datapoints IS attributes greater importance to feature 2 than feature 1 in the perturbed model. This shows that IS is sensitive to off-manifold perturbation. For ManifoldShap on the other hand, feature 2 is attributed greater importance for only about $10\%$ of the datapoints, less than all other baselines.
	
	We have also plotted the difference between estimated Shapley values and the ground truth IS values, for each value function. For a fair comparison between different value functions, we scale the Shapley values so that $\sum_{i\in \{1,2\}} |\phi_i | = 1$. As $\delta$ increases from 0 to 5, we can see that the errors in Shapley values increase for IS, while the errors in ManifoldShap are more concentrated around 0 than any other baseline.  

    The results show that ManifoldShap values, unlike IS, remain robust to off-manifold manipulations, while providing explanations which remain closer to ground truth IS values overall. CES and RJBShap, on the other hand can result in misleading explanations.

	\begin{figure*}[ht]
	\centering
	\begin{subfigure}{0.5\textwidth}
	    \centering
	    \includegraphics[height=0.98in]{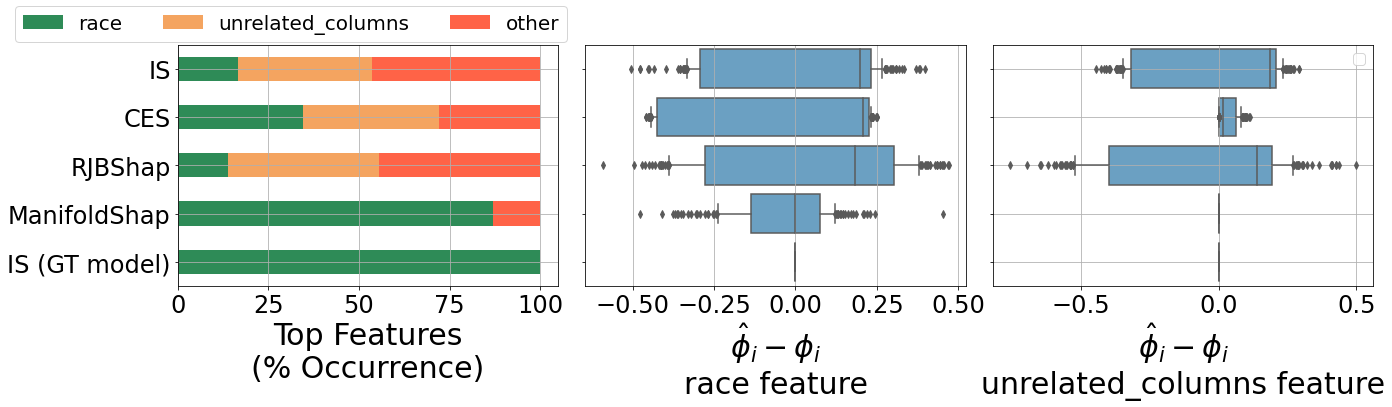}
	    \subcaption{COMPAS dataset results}
	    \label{fig:compas}
	\end{subfigure}%
	\begin{subfigure}{0.5\textwidth} 
	    \centering
	    \includegraphics[height=0.98in]{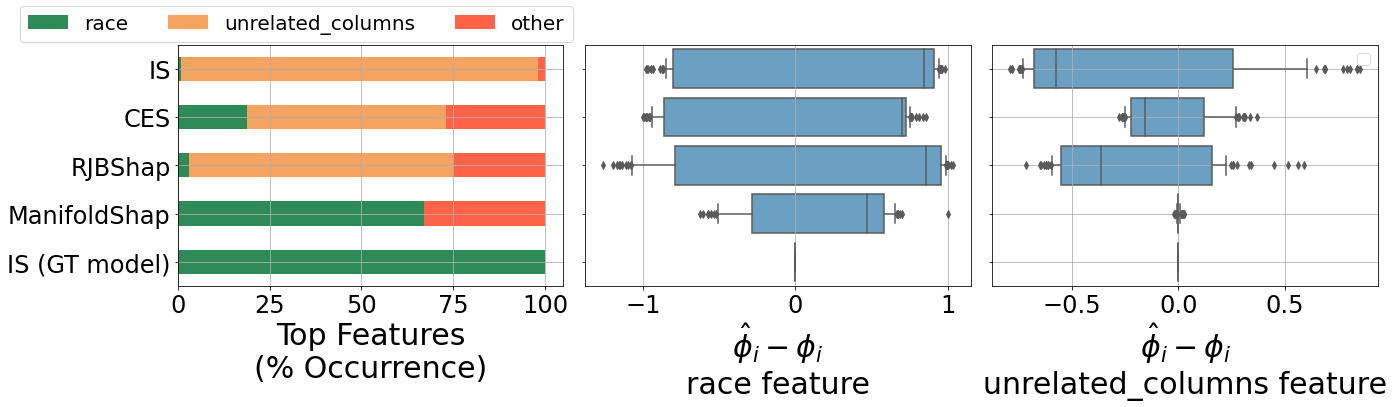}
	    \subcaption{Communities and Crime dataset results}
	    \label{fig:cc}
	\end{subfigure}
	\caption{Experiments on COMPAS and CC datasets. The barplots on the left of each subfigure shows the most important features for different Shapley values functions. The boxplots show the approximation errors of the Shapley values for different value functions\done{Rename CausalShap}.}\label{fig:real-world}
    \end{figure*}
    
    \done{IMP: Criticise RJBShap}
    
	\subsection{Real world datasets}

In this subsection, we evaluate the effect of adversarial off-manifold manipulation of models on Shapley values using real-world datasets. 
Specifically, using the same setup as in \citet{foolingshap}, we show that existing methodologies may fail to identify highly problematic model biases, whereas ManifoldShap can mitigate this problem due to its robustness properties.
 We consider the causal structure in Figure \ref{fig:dag} where the true features $\tilde{X}_i$ are distinguished from input features $X_i$, and therefore IS is equivalent to MS here. 
	
	\paragraph{Datasets.}
	The COMPAS dataset, collected by ProPublica \citep{machinebias}, includes information for 6172 defendants from Broward County, Florida. This information comprises 52 features including defendants' criminal history and demographic attributes. The sensitive attribute in this dataset is defendants' race. \done{Each defendant is classified based on whether they have a high-risk for recidivism.} The second dataset, Communities and Crime (CC), is a UCI dataset \citep{Dua:2019} which includes crime data in communities across the US, where each community constitutes a datapoint comprising 128 features. The sensitive attribute in CC is the percentage of Caucasian population. From here onwards, we use `race' to refer to the sensitive attribute for both datasets. \done{A label is assigned to each community which classifies if the proportion of violent crime in that community is above median or not.}
    \done{Next, we outline the experimental setup.}
	
	\paragraph{Biased classifier.} Following the strategy of \citet{foolingshap}, we construct the binary classifier $f$ to be only dependant on the sensitive feature for both datasets. Additional details are given in Appendix \ref{subsec:experimental_detals_app}. 

	\paragraph{Manifold estimation.} Just like in \citet{foolingshap}, we determine the manifold $\mathcal{Z}$ by training an OOD classifier. In particular, we follow the strategy in \citet{foolingshap} by perturbing each datapoint on randomly chosen features, and subsequently using these newly generated perturbations to train an OOD classifier\done{to detect OOD samples}. 
	
	\paragraph{Out of manifold perturbation.} To perturb the model outside the manifold $\mathcal{Z}$, we construct 2 synthetic features (referred to as `unrelated columns') like \citet{foolingshap}.
	For datapoints that lie outside $\mathcal{Z}$, only the `unrelated columns' are used to classify the datapoints. However, unlike \citet{foolingshap}, these `unrelated columns' are positively correlated with race. This is done to highlight a shortcoming of CES: even though CES is an on-manifold value function, the positive correlation between unrelated columns and race `fools' CES into attributing non-zero credit to the synthetic features. 
	\paragraph{Results.}

    We compute the Shapley values for the perturbed models on 500 datapoints from a randomly chosen held-out dataset. We use the supervised approach to estimate CES as outlined in Appendix \ref{subsec:CES-comp}.
    The barplots in Figures \ref{fig:compas} and \ref{fig:cc} show the percentage of data points in COMPAS and CC datasets respectively, for which each feature shows up as the top feature as per different value functions. For RJBShap, CES, and IS, there are more data points in both datasets with top feature among `unrelated columns' than data points with top feature of race. 
    For IS, this happens as a result of OOD perturbation of the model, and shows that when using IS, we can hide biases in the model by perturbing the model out of manifold. For RJBShap, this could be explained by the fact that it explicitly depends on the joint density $p(\x)$ of the data. Since, `unrelated columns' are positively correlated with race, the dependence of the density $p(\x)$ on these features and race is similar. As a result, `unrelated columns' get non-zero attributions in RJBShap.
    
    \done{For CES, this is also explained by the positive correlation between race and `unrelated columns'. Even though CES only evaluates the function on manifold, the positive correlation means that CES attributes similar importance for features `unrelated columns' as for race.} 
    \red{This positive correlation between race and `unrelated columns' also causes CES to attribute similar importance for features `unrelated columns' as for race.}
    This can be especially misleading when the data contains multiple correlated features which are not used by the model.
    
    On the other hand, for ManifoldShap, majority of the datapoints have top feature race, whereas none of them have top feature among `unrelated columns'. Figure \ref{fig:real-world} also shows the difference between estimated Shapley values and the ground truth IS values of the biased model. We have again rescaled the Shapley values so that $\sum_{i \in [d]}|\phi_i| = 1$ for fair comparison between different value functions. We can see that for the feature race, the errors of ManifoldShap are more concentrated around 0 than any other baseline considered. For `unrelated columns', ManifoldShap values are $\hat{\phi}_i=\phi_i=0$, i.e., ManifoldShap satisfies sensitivity property in this case. This shows that ManifoldShap is significantly more \done{more }robust to adversarial manipulation of the function outside the manifold, as well as robust to the attribution of credit based on correlations among features. \done{In this way, ManifoldShap retains the `causal interpretation' of Shapley values, while restricting the function evaluation to the data manifold.}

 \section{DISCUSSION AND LIMITATIONS}
	In this paper, we propose ManifoldShap, a Shapley value function which provides a compromise between existing on and off manifold value functions, by providing explanations which are robust to off-manifold perturbations of the model while estimating the causal contribution of features. However, ManifoldShap also has its limitations. 
	
	While our work does not make any assumptions on the set $\Z$, the properties of ManifoldShap are inherently linked to the choice of $\Z$. ManifoldShap is only robust to perturbation of model outside $\Z$ and perturbations inside $\Z$ could lead to significant changes in the computed Shapley values. It is therefore important to choose $\Z$ that is a good representative of the true data manifold, as otherwise, the Shapley values may not be robust to off-manifold perturbations. Additionally, as pointed out in Section \ref{sec:comparison}, restricting model evaluations to the set $\Z$ can reduce the causal accuracy of Shapley values. This becomes especially evident when the data manifold $\Z$ is \textit{sparse} or low-dimensional relative to the space $\mathcal{X}$. We highlight this empirically in Appendix \ref{subsec:corr}. Likewise, as we show in Appendix \ref{sec:properties}, the sensitivity and symmetry properties of ManifoldShap are also dependent on the properties of $\Z$. It is therefore worth exploring methodologies of choosing $\Z$ which provide the ideal trade-off between desirable properties like causal accuracy and robustness of explanations.
	We believe these limitations suggest interesting research questions that we leave for future work.

	\done{sensitivity axioms}
	\done{subspace robustness}
	\done{computing density can be difficult.}
	
	\done{in the case of sparse manifolds, we end up with conditional expectations.}
	
	\done{could think about the a softer version of ManifoldShap}

\subsubsection*{Acknowledgements}
We would like to thank Dominik Janzing for his valuable suggestions and insightful discussions. We are also grateful to Kailash Budhathoki and Philipp Faller for providing feedback on an earlier version of the manuscript.

\bibliography{ref}
\bibliographystyle{plainnat} %

\onecolumn
\appendix
\appendixpage

\startcontents[sections]
\printcontents[sections]{l}{1}{\setcounter{tocdepth}{2}}

\newpage

 \section{PROPERTIES OF MANIFOLDSHAP}\label{sec:properties}
In this section, we consider the theoretical properties of ManifoldShap. 
 The proofs for results in this section are provided in Section \ref{proofs}.
	\subsubsection{Sensitivity Property}
 The following result holds in the setting of \citet{lshap}, i.e., when the real feature values are formally distinguished from the feature values input into the function (see Figure \ref{fig:dag}). In this case, the interventional distribution of $p(\X_{\bar{S}} \mid do(\X_S = \x_S))$ is the same as the marginal distribution $p(\X_{\bar{S}})$.
	\begin{proposition}[Sensitivity]\label{sensitivity}
	Let $i \in [d]$ be such that 
		\begin{enumerate}
			\item the function $f(\textbf{x})$ does not depend on $x_i$ for all values of $\textbf{x}$, 
			\item the set $\Z$ is of the form, $\Z = \Z_1 \times \dots \times \Z_d$, where $\Z_j \subseteq \mathbb{R}$ and $X_i \in \Z_i$ almost surely.
		\end{enumerate}
		Then, if the causal graph of features is as shown in Figure \ref{fig:dag}, we have that $\valgeneric{f}{\Z}(S) = \valgeneric{f}{\Z}(S \cup \{i\})$, and therefore $\atr{i}{f} = 0$. 
	\end{proposition}

	\textbf{Remark.} As mentioned previously, in this paper we argue that a function should mainly be characterised by it's behaviour on manifold. Note that in this case, unlike the classical formulation of Sensitivity axiom \citep{kernelshap}, we also need the condition 2 above, which implies that $\ind(\X \in \Z)$ is independent of $X_i$.
	This can be justified as follows: Define a function $h(\textbf{x}) \coloneqq \ind(\x \in \Z)f(\textbf{x})$. If condition 2 does not hold, i.e., $\ind(\X \in \Z)$ depends on $X_i$, then $h(\textbf{X})$ must depend on $X_i$. Moreover, by definition, $h(\textbf{x}) = f(\textbf{x})$ for all $\x \in \Z$, i.e. $h$ and $f$ agree on $\Z$. Therefore, since $f$ and $h$ are indistinguishable on the $\Z$, and $h(\X)$ depends on $X_i$, it would be misleading to have zero attribution for the $i$'th feature.
	
	\blue{Comment about how in practice, sensitivity assumption tends to hold under relatively weak conditions.}\blue{Emphasize that condition 2 above is much weaker than independence of features.}
	
	\subsubsection{Symmetry Property}
 Like the previous result, the following result holds in the setting of \citet{lshap}, where Interventional Shapley is equivalent to Marginal Shapley.
	\begin{proposition}[Symmetry]\label{symmetry}
		Let $i, j \in [d]$ be such that 
		\begin{enumerate}
			\item the function $f(\textbf{x})$ is symmetric in components $i$ and $j$ on $\Z$,
			\item the density $p(\textbf{x})$ is symmetric in components $i$ and $j$,
			\item the function $\ind(\x \in \Z)$ is symmetric in components $i$ and $j$.
		\end{enumerate}
		Then, if the causal graph of features is as shown in Figure \ref{fig:dag}, we have that for any $S \subseteq [d] \setminus \{i, j\}$ and $\textbf{x}$ such that $x_i = x_j$, we have that $\valgeneric{f}{\Z}(S \cup \{i\}) = \valgeneric{f}{\Z}(S \cup \{j\})$, and therefore $\atr{i}{f} = \atr{j}{f}$.
	\end{proposition}

    \textbf{Remark.} The condition 3 above states that $\Z$ should be symmetric in components $i$ and $j$. This condition will be satisfied if, for example, $\Z$ is a ball centred at origin. Moreover, if condition 2 is satisfied, i.e., the density $p(\x)$ is symmetric in $x_i$ and $x_j$, it is straightforward to show that the $\epsilon$-density manifold will also satisfy condition 3, for any $\epsilon>0$. Additionally, we emphasise that condition 2 is not specific to our value function, and is also needed for symmetry property to hold for CES and IS. \citet{lshap} illustrates this with an example where symmetry fails to hold for both CES and IS without condition 2. \blue{make sure this is true}
	
	\subsubsection{Efficiency Property}
 All results presented in this section from here onwards do not assume a specific causal structure on the features and hold for any general causal graph on features.
	\begin{proposition}[Efficiency]
		The value function $\valgeneric{f}{\Z}(S) = \valfunset{S}{f}{\Z}$ satisfies 
		\[
		\valgeneric{f}{\Z}([d]) - \valgeneric{f}{\Z}(\emptyset) = f(\textbf{x}) - \E[f(\X) \mid \X \in \Z] ,
		\]
		for any $\textbf{x} \in \Z$.
		Therefore, $\sum_i \atr{i}{f} = f(\textbf{x}) - \E[f(\X) \mid \X \in \Z]$.
	\end{proposition}
	\begin{proof}
		Follows straightforwardly from the definition of $\valgeneric{f}{\Z}(S)$.
	\end{proof}
	
	\subsubsection{Linearity Property}
	\begin{proposition}[Linearity]
		For any functions $f_1, f_2$ and $\alpha_1, \alpha_2 \in \mathbb{R}$, 
		\[
		\valgeneric{\alpha_1 f_1 + \alpha_2 f_2}{\Z}(S) = \alpha_1 \valgeneric{f_1}{\Z}(S) + \alpha_2 \valgeneric{f_2}{\Z}(S)
		\]
		Therefore, $\atr{i}{\alpha_1 f_1 + \alpha_2 f_2}^{\alpha_1 f_1 + \alpha_2 f_2} = \alpha_1\atr{i}{f_1}^{f_1}  + \alpha_2\atr{i}{f_2}^{f_2}$, where $\atr{i}{f}^f$ denotes the Shapley value for feature $i$ and function $f$.
	\end{proposition}
	\begin{proof}
		Follows straightforwardly from the definition of $\valgeneric{f}{\Z}(S)$ and the linearity of the expectation $\valfunset{S}{f}{\Z}$.
	\end{proof}
	
\subsection{Robustness and Causal Accuracy of ManifoldShap}\label{subsec:tv-distance}
	\blue{this holds for general DAGs}
We consider the family of value functions of the form $v_{f, p_S}(S) = \E_{\X \sim p_S}[f(\X)]$ for some measure $p_S$. Both, Interventional Shapley and ManifoldShap are part of this family with corresponding densities $p^{\textup{do}}_{\x_S} (\y) \coloneqq p(\y \mid do(\X_S=\x_S))$ and $p_{\Z, \x_S}(\y) \coloneqq \frac{p(\y \mid do(\X_S=\x_S)) \ind(\y\in\Z)}{\p(\X \in \Z \mid do(\X_S=\x_S))}$ respectively. Next, we show that $p_{\Z, \x_S}$ minimises the total variation distance with the interventional distribution $p^{\textup{do}}_{\x_S}$ while satisfying subspace robustness in definition \ref{subspacerobustness}.
\begin{proposition}\label{tv-distance}
    The measure $p_{\Z, \x_S}$ satisfies
    \begin{align*}
        p_{\Z, \x_S} \in \arg\min_{p_S} \{\textup{TV}(p_S, p^{\textup{do}}_{\x_S}): v_{f, p_S} \textup{  is strong T-robust on subspace $\Z$} \}
    \end{align*}
\end{proposition}

\begin{proof}
If $v_{f, p_S}$ is strong T-robust on subspace $\Z$, then consider functions $f_1, f_2$ satisfying \[
    f_2(\x) \coloneqq f_1(\x) + \delta \ind(\x\in\Z) + K \ind(\x\not\in\Z)
\] 
for some $\delta, K > 0$. Then, $\max_{\x\in\Z} |f_1(\x) - f_2(\x) | = \delta$. Moreover, 
\begin{align*}
    | v_{f_1, p_S}(S) - v_{f_2, p_S}(S)| =& | \E_{\X \sim p_S}[f_1(\X)] - \E_{\X \sim p_S}[f_2(\X)]|\\
    =& \delta p_S(\X \in Z) + K p_S(\X \not\in Z)
\end{align*}
Since we can pick $K$ to be arbitrarily large, $v_{f, p_S}$ satisfies strong T-robustness on subspace $\Z$ only if $p_S(\X \not\in Z) = 0$.

Next, note that if $p_S(\X \in\Z) = 1$, 
\begin{align*}
    &\textup{TV}(p_S, p^{\textup{do}}_{\x_S}) \\
    =& 1/2 \int_{\y} | p_S(\y) - p^{\textup{do}}_{\x_S}(\y) | \mathrm{d}\y\\
    =& 1/2 \int_{\y} | p_S(\y) - p(\y \mid do(\X_S=\x_S)) | \mathrm{d}\y\\
    =& 1/2 \int_{\y \in \Z }| p_S(\y) - p(\y \mid do(\X_S=\x_S)) | \mathrm{d}\y + 1/2 \int_{\y \not\in \Z }| p_S(\y) - p(\y \mid do(\X_S=\x_S)) | \mathrm{d}\y\\
    \geq& 1/2 \left|\int_{\y \in \Z } p_S(\y) -  p(\y \mid do(\X_S=\x_S)) \mathrm{d}\y \right| + 1/2 \left|\int_{\y \not\in \Z } p_S(\y) -  p(\y \mid do(\X_S=\x_S)) \mathrm{d}\y \right| \\
    =& 1/2 \left| p_S(\X\in\Z) -  \p(\X\in\Z \mid do(\X_S=\x_S))\right| + 1/2| p_S(\X\not\in\Z) - \p(\X \not\in \Z \mid do(\X_S=\x_S))|\\
    =& 1/2  \left(1 - \p(\X \in \Z \mid do(\X_S=\x_S))\right) + 1/2 \p(\X \not\in \Z \mid do(\X_S=\x_S)) \\
    =& 1/2 \int_{\y \in \Z} \left| \frac{p(\y \mid do(\X_S=\x_S))\ind(\y\in\Z)}{\p(\X \in \Z \mid do(\X_S=\x_S))}  - p(\y \mid do(\X_S=\x_S)) \right| \mathrm{d}\y + 1/2 \p(\X \not\in \Z \mid do(\X_S=\x_S))\\
    =& 1/2 \int_{\y} \left| \frac{p(\y \mid do(\X_S=\x_S))\ind(\y\in\Z)}{\p(\X \in \Z \mid do(\X_S=\x_S))}  - p(\y \mid do(\X_S=\x_S)) \right|\mathrm{d}\y \\
    =& \textup{TV}(p_{\Z, \x_S}, p^{\textup{do}}_{\x_S})
\end{align*}
\end{proof}
Proposition \ref{tv-distance} shows that among all the value functions of the form $v_{f, p_S}$ which are strong T-robust on subspace $\Z$, ManifoldShap provides the \textit{best} approximation to Interventional Shapley values. This further highlights that ManifoldShap provides a compromise between on and off manifold value functions -- it satisfies subspace robustness while also approximating causal contribution of features.

\newpage
\section{PROOFS}\label{proofs}
\paragraph{Proof of Lemma \ref{manifoldShap}.}
\begin{proof}
Using the definition of ManifoldShap, we get that 
\begin{align*}
    \valgeneric{f}{\Z}(S) =& \int_{\y} f(\y) p_{\Z, \x_S}(\y) \mathrm{d}\y\\
    =& \int_{\y} f(\y) \frac{p(\y \mid do(\X_S=\x_S))\ind(\y \in \Z) }{\p(\X \in \Z \mid do(\X_S=\x_S))} \mathrm{d}\y \\
    =& \frac{1 }{\p(\X \in \Z \mid do(\X_S=\x_S))} \int_{\y} f(\y)\ind(\y \in \Z) p(\y \mid do(\X_S=\x_S))\mathrm{d}\y \\
    =& \frac{\E[f(\X)\ind(\X \in \Z) \mid do(\X_S=\x_S)]}{\p(\X \in \Z \mid do(\X_S=\x_S))}.
\end{align*}
\end{proof}
\paragraph{Proof of Proposition \ref{t-robust}.}
	\begin{proof}
	\begin{align}
		&\max_{\textbf{x}} | f_1(\textbf{x}) - f_2(\textbf{x}) | p(\textbf{x}) \leq \delta \nonumber \\
		&\implies \sup_{\textbf{x} \in \man} | f_1(\textbf{x}) - f_2(\textbf{x}) | p(\textbf{x}) \leq \delta \nonumber \\
		&\implies \sup_{\textbf{x} \in \man} | f_1(\textbf{x}) - f_2(\textbf{x}) | \epsilon \leq \sup_{\textbf{x} \in \man} | f_1(\textbf{x}) - f_2(\textbf{x}) | p(\textbf{x}) \leq \delta \nonumber\\
		&\implies \sup_{\textbf{x} \in \man} | f_1(\textbf{x}) - f_2(\textbf{x}) |  \leq \delta/\epsilon \nonumber
	\end{align}
	Using the above, 
	\begin{align}
		| \valgeneric{f_1}{\man}(S) - \valgeneric{f_2}{\man}(S) | =& |\valfunset{S}{f_1}{\man} - \valfunset{S}{f_2}{\man} |  \nonumber\\
		\leq& \E[| f_1(\X) - f_2(\X) | \mid do(\X_S = \x_S), \X \in \man]\nonumber \\
		\leq& \sup_{\textbf{x} \in \man} | f_1(\textbf{x}) - f_2(\textbf{x}) |  \leq \delta/\epsilon \nonumber
	\end{align}
\end{proof}

\paragraph{Proof of Proposition \ref{manshap-subspace-robustness}.}
\begin{proof}
Let $\sup_{\x\in\Z'} |f_1(\x) - f_2(\x)|\leq \delta$. Then, for any $S\subseteq [d]$,
\begin{align*}
    |\valgeneric{f_1}{\Z}(S) -\valgeneric{f_2}{\Z}(S)  | =& |\valfunset{S}{f_1}{\Z}  - \valfunset{S}{f_2}{\Z}| \\
    \leq& \E [| f_1(\X) - f_2(\X) | \mid do(\X_S = \x_S), \X \in \Z ] \\
    \leq& \sup_{\x\in\Z}|f_1(\x) - f_2(\x)|\\
    \leq& \sup_{\x\in\Z'}|f_1(\x) - f_2(\x)|\leq \delta. 
\end{align*}
\end{proof}

\paragraph{Proof of Proposition \ref{subspace-robustness-causalshap}.}
\begin{proof}
    Let $S = \emptyset$, then $v^{\textup{IS}}_{\textbf{x}, f_1}(S) = v^{\textup{CES}}_{\textbf{x}, f_1}(S) = v^{\textup{MS}}_{\textbf{x}, f_1}(S) = \E[f(\X)]$. Let $f_2(\x) \coloneqq f_1(\x) + K\ind(\x\not\in\Z')$ for some  $K>0$. Then, we have that $\sup_{\x\in\Z'} |f_1(\x) - f_2(\x)|=0$. Moreover,
    \begin{align*}
        |\E[f_1(\X)] - \E[f_2(\X)]| = |K \E[\ind(\X\not\in \Z')]| = K\p(\X\not\in\Z') > 0.
    \end{align*}
    Since we can choose $K$ to be arbitrarily big, it follows that $|\E[f_1(\X)] - \E[f_2(\X)]|$ is not bounded for general functions $f_1, f_2$ satisfying $\sup_{\x\in\Z'} |f_1(\x) - f_2(\x)|\leq \delta$.
    
    Now, for JBShap and RJBShap, define $f_1, f_2$ such that $f_2(\x) \coloneqq f_1(\x) + K\ind(\x\not\in\Z', p(\x)>0)/p(\x)$. Then, 
\[
\sup_{\x\in\Z'} | f_1(\x) - f_2(\x)| = 0.
\]

Let $\x \in \mathbb{R}^d$ be such that $\x \not\in \Z'$ and $p(\x)>0$. Since $\p(\X \in \Z') < 1$, there must exist an $\x \in \mathbb{R}^d$ which satisfies this condition. Then for $S=\emptyset$,
\begin{align*}
    | v^{\textup{J}}_{\x, f_1, p}(S) - v^{\textup{J}}_{\x, f_2, p}(S) | =& |f_1(\x)p(\x) - f_2(\x)p(\x)| \\
    =& K | \ind(\x\not\in\Z', p(\x)>0)| = K.
\end{align*}
Since we can choose $K$ to be arbitrarily big, it follows that $| v^{\textup{J}}_{\x, f_1, p}(\emptyset) - v^{\textup{J}}_{\x, f_2, p}(\emptyset) |$ is not bounded for general functions $f_1, f_2$ satisfying $\sup_{\x\in\Z'}|f_1(\x) - f_2(\x)|\leq \delta$. 

Moreover, we have that for $S = \emptyset$,
\begin{align*}
    | v^{\textup{RJ}}_{\x, f_1, p}(S) - v^{\textup{RJ}}_{\x, f_2, p}(S) | =& |\E[f_1(\X)p(\X)] - \E[f_2(\X)p(\X)]| \\
    =& K | \E[\ind(\X\not\in\Z', p(\X) > 0)] | \\
    =& K | \E[\ind(\X\not\in\Z')] | \\
    =& K \p(\X\not\in\Z') > 0.
\end{align*}
Since we can choose $K$ to be arbitrarily big, it follows that $|v^{\textup{RJ}}_{\x, f_1, p}(\emptyset) - v^{\textup{RJ}}_{\x, f_2, p}(\emptyset)| $ is not bounded for general functions $f_1, f_2$ satisfying $\sup_{\x\in\Z'}|f_1(\x) - f_2(\x)|\leq \delta$. 
\end{proof}

\paragraph{Proof of Proposition \ref{sensitivity}.}
	\begin{proof}
	Recall that in the setting we are considering, the interventional distribution $p(\X_{\bar{S}}\mid do(\X_S = \x_S))$ is equal to the marginal distribution $p(\X_{\bar{S}})$.
	
    Let $S$ be such that $i \not\in S$, and let $\x \in \Z$ be any point. Then, $\ind((\x_S, \X_{\bar{S}}) \in \Z) = \prod_{j \in \bar{S}}\ind(X_j \in \Z_j)$. Using the fact that $\ind(X_i \in \Z_i)\overset{\textup{a.s.}}{=}1$, we get that $\ind(\X_{\bar{S}} \in \Z) \overset{\textup{a.s.}}{=}\prod_{j \in \bar{S}\setminus \{i\}}\ind(X_j \in \Z_j)$.
	\begin{align}
		\valgeneric{f}{\Z}(S) &= \frac{\E[f(\textbf{x}_S, \textbf{X}_{\bar{S}}) \ind((\textbf{x}_S, \textbf{X}_{\bar{S}}) \in \Z) ] }{\E[ \ind((\textbf{x}_S, \textbf{X}_{\bar{S}}) \in \Z)]} \nonumber \\
		&= \frac{\E[f(\textbf{x}_S, \textbf{X}_{\bar{S}}) \prod_{j \in \bar{S}}\ind(X_j \in \Z_j) ] }{\E[ \prod_{j \in \bar{S}}\ind(X_j \in \Z_j) ]} \nonumber \\
		&= \frac{\E[f(\textbf{x}_S, \textbf{X}_{\bar{S}}) \prod_{j \in \bar{S}\setminus\{i\}}\ind(X_j \in \Z_j) ] }{\E[ \prod_{j \in \bar{S}\setminus\{i\}}\ind(X_j \in \Z_j) ]} \nonumber \\
        &= \frac{\E[f(\textbf{x}_{S\cup \{i\}}, \textbf{X}_{\bar{S} \setminus \{i\}}) \prod_{j \in \bar{S}\setminus\{i\}}\ind(X_j \in \Z_j) ] }{\E[ \prod_{j \in \bar{S}\setminus\{i\}}\ind(X_j \in \Z_j) ]} \nonumber \\
		&= \frac{\E[f(\textbf{x}_{S \cup \{i\}} , \textbf{X}_{\bar{S}\setminus\{i\} }) \ind((\textbf{x}_{S \cup \{i\}}, \textbf{X}_{\bar{S}\setminus\{i\}}) \in \Z)  ] }{\E[ \ind((\textbf{x}_{S \cup \{i\}}, \textbf{X}_{\bar{S}\setminus\{i\}}) \in \Z) ]} = \valgeneric{f}{\Z}(S\cup \{i\}) \nonumber 
	\end{align}
	where, in the second last step above we use the fact that, $f(\textbf{x})$ is independent of $x_i$.
	
	\blue{this only works in the setting of \citet{lshap}. Make this more general.}
\end{proof}

\paragraph{Proof of Proposition \ref{symmetry}.}
	\begin{proof}
	\textbf{Notation:} Let $m(\x)$ be any function. We use the notation $m(\X_{S} = \x_S, \X_{\bar{S}} = \x'_{\bar{S}})$ to explicitly denote $m(\x_S, \x'_{\bar{S}})$.
	
	Suppose $m(\textbf{x})$ is a function symmetric in components $i$ and $j$. Then, if $S \subseteq [d] \setminus \{i, j\}$. Then, 
	\begin{align}
		&\E[m(\textbf{x}_{S \cup \{i\}}, \textbf{X}_{\bar{S}\setminus \{i\}})] \nonumber \\ 
		=& \int_{\textbf{X}'_{\bar{S}\setminus \{i\}}} m(\textbf{X}_{S \cup \{i\}} = \textbf{x}_{S \cup \{i\}}, \textbf{X}_{\bar{S}\setminus \{i\}} =\textbf{X}'_{\bar{S}\setminus \{i\}}) \int_{\textbf{Y}_{S \cup \{i\}}} p(\textbf{X}_{S \cup \{i\}} = \textbf{Y}_{S \cup \{i\}}, \textbf{X}_{\bar{S}\setminus \{i\}}=\textbf{X}'_{\bar{S}\setminus \{i\}}) \mathrm{d}\textbf{Y}_{S \cup \{i\}} \mathrm{d}\textbf{X}'_{\bar{S}\setminus \{i\}} \nonumber \\
		=& \int_{\textbf{X}'_{\bar{S}\setminus \{i\}}} m(\textbf{X}_{S \cup \{j\}} = \textbf{x}_{S \cup \{j\}}, \textbf{X}_i = \textbf{X}'_j, \textbf{X}_{\bar{S}\setminus \{i,j\}} =\textbf{X}'_{\bar{S}\setminus \{i,j\}}) \nonumber\\
		&\times \int_{\textbf{Y}_{S \cup \{i\}}} p(\textbf{X}_{S} = \textbf{Y}_{S}, \textbf{X}_j = \textbf{Y}_i, \textbf{X}_i=\textbf{X}'_j,  \textbf{X}_{\bar{S}\setminus \{i, j\}}=\textbf{X}'_{\bar{S}\setminus \{i,j\}}) \mathrm{d}\textbf{Y}_{S \cup \{i\}} \mathrm{d}\textbf{X}'_{\bar{S}\setminus \{i\}} \nonumber 
	\end{align}
	where, in the last step above we use the fact that both, $m(\textbf{x})$ and $p(\textbf{x})$ are symmetric in components $i$ and $j$ and $x_i = x_j$. Next, relabelling the dummy variables $\textbf{X}'_j$ as $\textbf{X}'_i$ and $\textbf{Y}_i$ as $\textbf{Y}_j$, the above becomes
	\begin{align}
		&\int_{\textbf{X}'_{\bar{S}\setminus \{j\}}} m(\textbf{X}_{S \cup \{j\}} = \textbf{x}_{S \cup \{j\}}, \textbf{X}_{\bar{S}\setminus \{j\}} =\textbf{X}'_{\bar{S}\setminus \{j\}}) \int_{\textbf{Y}_{S \cup \{j\}}} p(\textbf{X}_{S \cup \{j\}} = \textbf{Y}_{S \cup \{j\}}, \textbf{X}_{\bar{S}\setminus \{j\}}=\textbf{X}'_{\bar{S}\setminus \{j\}}) \mathrm{d}\textbf{Y}_{S \cup \{j\}} \mathrm{d}\textbf{X}'_{\bar{S}\setminus \{j\}} \nonumber\\
		&= \E[m(\textbf{x}_{S \cup \{j\}}, \textbf{X}_{\bar{S}\setminus \{j\}})] \nonumber
	\end{align}

	Next, we use the fact that the functions $m_1(\x) \coloneqq \ind(\x\in\Z)$ and $m_2(\textbf{x}) \coloneqq f(\textbf{x}) \ind(\textbf{x} \in \Z)$ are symmetric in components $i$ and $j$.
	Therefore, using the result above, we get that,
	\begin{align}
		\valgeneric{f}{\Z}(S\cup \{i\}) =& \frac{\E[f(\textbf{x}_{S \cup \{i\}} , \textbf{X}_{\bar{S}\setminus\{i\} }) \ind((\textbf{x}_{S \cup \{i\}}, \textbf{X}_{\bar{S}\setminus\{i\}}) \in \Z) ] }{\E[ \ind((\textbf{x}_{S \cup \{i\}}, \textbf{X}_{\bar{S}\setminus\{i\}}) \in \Z)]} \nonumber \\
		=& \frac{\E[m(\textbf{x}_{S \cup \{i\}} , \textbf{X}_{\bar{S}\setminus\{i\} })]}{\E[ \ind((\textbf{x}_{S \cup \{i\}}, \textbf{X}_{\bar{S}\setminus\{i\}}) \in \Z)]} \nonumber \\
		=& \frac{\E[m(\textbf{x}_{S \cup \{j\}} , \textbf{X}_{\bar{S}\setminus\{j\} })]}{\E[ \ind((\textbf{x}_{S \cup \{j\}}, \textbf{X}_{\bar{S}\setminus\{j\}}) \in \Z)]} \nonumber \\
		=& \valgeneric{f}{\Z}(S\cup \{j\}) \nonumber
	\end{align}
\end{proof}

\begin{proposition}\label{optimality}
    Let $\alphaman$ be as defined in Def \ref{mass-manifold} and let $\Z$ be any set with $\p(\X\in\Z) \geq  \p(\X\in\alphaman)$. Then, if $\epsilon^{(\alpha)} > 0$, we have that $|\Z| \geq |\alphaman|$, where $|\mathcal{S}| \coloneqq \int_{\mathcal{S}}\mathrm{d}\x$.
\end{proposition}

\paragraph{Proof of Proposition \ref{optimality}.}
\begin{proof}
\begin{align*}
    |\Z| - |\alphaman| =& \left(|\Z \setminus \alphaman | +|\Z \cap \alphaman |\right) -  \left(|\alphaman \setminus \Z | +|\Z \cap \alphaman  | \right) \\
    =& |\Z \setminus \alphaman | - |\alphaman \setminus \Z |
\end{align*}
Similarly, 
\begin{align*}
    0 \leq& \p(\X\in \Z) - \p(\X\in\alphaman) \\
    =& \left(\p(\X \in \Z \setminus \alphaman) + \p(\X \in \Z \cap \alphaman) \right) - \left(\p(\X \in \alphaman \setminus \Z) + \p(\X \in \Z \cap \alphaman) \right)\\
    =& \p(\X \in \Z \setminus \alphaman) - \p(\X \in \alphaman \setminus \Z) \\
    =& \int_{\Z \setminus \alphaman} p(\x) \mathrm{d}\x - \int_{\alphaman \setminus \Z} p(\x) \mathrm{d}\x \\
    \leq& \int_{\Z \setminus \alphaman} \epsilon^{(\alpha)} \mathrm{d}\x - \int_{\alphaman \setminus \Z} \epsilon^{(\alpha)} \mathrm{d}\x \\
    =& \epsilon^{(\alpha)} \left(|\Z \setminus \alphaman | - |\alphaman \setminus \Z |\right)
\end{align*}
In the second last step above, we use the fact that $p(\x) \geq \epsilon^{(\alpha)} \iff \x \in \alphaman$. Using the condition $\epsilon^{(\alpha)}>0$, we get that
\begin{align*}
    |\Z| - |\alphaman| = |\Z \setminus \alphaman | - |\alphaman \setminus \Z | \geq 0
\end{align*}
\end{proof}

\newpage
\section{ALTERNATIVE METHODOLOGIES OF COMPUTING MANIFOLDSHAP}\label{subsec:manshap-alternative-methods}
In this section, we outline alternative methodologies of computing ManifoldShap value function. As before, we assume that we can sample from the interventional distribution $p(\X_{\bar{S}} \mid do(\X_S = \x_S))$ for any $S\subseteq [d]$. This is a standard assumption needed to estimate Interventional Shapley.
\subsection{Supervised approach}\blue{fix the notations in this section}
Here, we use the fact that the expectation $g(\x_S) \coloneqq \valfunset{S}{f}{\Z}$ minimises the mean squared error $\mathcal{L}_S(h) = \E_{\tilde{\X}_S \sim p(\X_S), \tilde{\X} \sim p(\X \mid do(\X_S = \tilde{\X}_S), \X\in\Z)}[f(\tilde{\X}) - h(\tilde{\X}_S)]^2$. Using this, we can define a surrogate model $g_{\theta}(\x_S)$ that takes as input coalition of features $\x_S$ (e.g., by masking features in $\bar{S}$) and that is trained to minimise the loss:
\begin{align*}
    \mathcal{L}(\theta) = \E_{\tilde{\X}_S \sim p(\X_S), \tilde{\X} \sim p(\X \mid do(\X_S = \tilde{\X}_S), \X \in \Z)}\E_{S \sim \textup{Shapley}}[f(\tilde{\X}) - g_\theta(\tilde{\X}_S)]^2
\end{align*}
Here, $S\sim \textup{Shapley}$ corresponds to sampling coalitions from the distribution where the probability assigned to each coalition is the combinatorial factor $|S|!(n - |S| - 1)!/n!$. Additionally, rejection sampling can be used to sample $\tilde{\X} \sim p(\X \mid do(\X_S = \tilde{\X}_S), \X \in \Z)$. To be specific, we repeated sample $\tilde{\X} \sim p(\X \mid do(\X_S = \tilde{\X}_S))$ until the sampled value $\tilde{\X}$ lies in $\Z$.

\subsection{Rejection sampling}\label{subsec:rejection_sampling}
In this subsection, we extend the approach in \citet{shap1} to propose an efficient sampling-based approximation. This approximation uses the following alternative formulation of Shapley values:
\begin{align*}
    \phi_i = \sum_{\pi \in \Pi} \frac{1}{n!} \left[v(\{j: \pi(j)\leq \pi(i) \}) - v(\{j: \pi(j)< \pi(i) \}) \right]
\end{align*}
where $\Pi$ denotes the set of permutations of $N$, and $\pi(j)<\pi(i)$ means that $j$ precedes $i$ under ordering $\pi$. To derive the sampling procedure, we observe that the Shapley value for feature $i$, $\pi_i$, can be written as an average over the set of permutations, i.e., 
\begin{align*}
    \phi_i = \E_{\pi}[v(\{j: \pi(j)\leq \pi(i) \}) - v(\{j: \pi(j)< \pi(i) \})]
\end{align*}
where the permutations $\pi$ are drawn from a uniform distribution over $\Pi$. Using this, we derive the following procedure for obtaining an unbiased and consistent estimation of ManifoldShap.

\begin{algorithm}
\caption{Approximating ManifoldShap value $\phi_i$ for instance $\x \in \Z$.}\label{alg:cap}
\hspace*{\algorithmicindent} \textbf{Input:} Instance $\x$; the desired number of samples $m$; feature $i$ to compute Shapley value for;
\begin{algorithmic}
\State $\phi_i \gets 0$
\For{$j=1$ to $m$}
\State choose a random permutation of features $\pi \in \Pi$
\State $S \gets \{j: \pi(j)< \pi(i) \}$
\State sample $\y_{\bar{S}\setminus \{i\}} \sim p(\X_{\bar{S}\setminus \{i\}} \mid do(\X_{S\cup \{i\}} = \x_{S\cup \{i\}}))$
\While{$(\x_{S\cup \{i\}}, \y_{\bar{S}\setminus \{i\}}) \not \in \Z$}
\State sample $\y_{\bar{S}\setminus \{i\}} \sim p(\X_{\bar{S}\setminus \{i\}} \mid do(\X_{S\cup \{i\}} = \x_{S\cup \{i\}}))$
\EndWhile
\State sample $\textbf{z}_{\bar{S}} \sim p(\X_{\bar{S}} \mid do(\X_S = \x_S))$
\While{$(\x_{S}, \textbf{z}_{\bar{S}}) \not \in \Z$}
\State sample $\textbf{z}_{\bar{S}} \sim p(\X_{\bar{S}} \mid do(\X_S = \x_S))$
\EndWhile
\State $\phi_i \gets \phi_i + (f(\x_{S\cup \{i\}}, \y_{\bar{S}\setminus \{i\}}) - f(\x_{S}, \textbf{z}_{\bar{S}}))$
\EndFor
\State $\phi_i \gets \phi_i/m$\\
\textbf{Return:} $\phi_i$
\end{algorithmic}
\end{algorithm}

\newpage
\section{INTERVENTIONAL SHAPLEY VS CONDITIONAL EXPECTATION SHAPLEY}\label{sec:int-vs-ces}

    \paragraph{Example.}
	Assume that $\mathcal{X} = \{0,1\}^2$, and that the features $X_1, X_2$ follow the causal structure shown below. In this setting, interventional distributions are equivalent to marginal distributions, i.e., $p(\X_{\bar{S}} \mid do(\X_S = \x_S)) = p(\X_{\bar{S}})$.\\

	\begin{center}
	\begin{tikzpicture}
	\tikzset{
    > = stealth,
    every node/.append style = {
        draw = black,
        shape = circle,
        inner sep = 0.5pt,
        minimum size=0.75cm
    },
    every path/.append style = {
        arrows = ->,
    }
    }
    \tikz{
        \node (x1) at (0,1) {$X_1$};
        \node (x2) at (2,1) {$X_2$};
        \node[fill=gray] (z) at (1,2) {\textcolor{white}{$Z$}};
        \node (y) at (1,0) {$Y$};
        \path (z) edge (x1);
        \path (z) edge (x2);
        \path (x1) edge (y);
        \path (x2) edge (y);
    }
    \end{tikzpicture}
    \end{center}
	
	Consider the case where $f(x_1, x_2)=x_1$ and $Z, X_1, X_2$ are binary variables, with 
	\begin{align*}
	    Z =& \begin{cases}
			0 & \textup{w.p. 0.5} \\
			1 & \textup{otherwise}
			\end{cases} \, \\
		X_1 =& Z \, \\
		X_2 =& \begin{cases}
			Z & \textup{w.p. $p$ (for some $p>0$),} \\
			1-Z & \textup{otherwise.}
		\end{cases}
	\end{align*}
	 In this case, $\E[f(X_1, X_2)\mid do(X_2 = x_2)] = \E[f(X_1, X_2)] = \E[f(X_1, x_2)] = 1/2$ and 
	 \[
	 \E[f(X_1, X_2)\mid do(X_1 = x_1, X_2=x_2)] = \E[f(X_1, X_2)\mid do(X_1 = x_1)] = \E[f(x_1, X_2)] = x_1,
	 \]
	 for any $x_2$. 
	 It straightforwardly follows that, in this case, the Interventional Shapley value for feature $x_2$, $\phi_2 = 0$, i.e. Interventional Shapley satisfies the Sensitivity property. 
	
	However, if we use CES instead, we get that 
	\begin{align*}
	    \phi_2 =& \frac{1}{2}\left( \E[f(X_1, X_2)\mid X_2=x_2]- \E[f(X_1, X_2)] + \E[f(X_1, X_2)\mid X_1 = x_1 ,X_2=x_2]- \E[f(X_1, X_2)\mid X_1=x_1] \right)\\
	    =& \frac{1}{2}\left(\E[X_1\mid X_2=x_2] - 1/2 + x_1 - x_1 \right)\\
	    =& \frac{1}{2}\left(\p(X_1=1\mid X_2=x_2) - 1/2 \right)\\
	    =& \frac{1}{2}\left(\frac{1/2 p \ind(x_2 = 1) + 1/2 (1-p) \ind(x_2 = 0)}{1/2} - 1/2 \right)\\
	    =& \frac{1}{2} \left(p\ind(x_2 = 1) + (1-p)\ind(x_2 = 0) - 1/2\right). 
	\end{align*}
	which is non-zero when $p\neq 1/2$. 
	
	This example illustrates that CES value function can lead to misleading Shapley values, especially when the features are highly correlated. Interventional Shapley value function, on the other hand, incorporates the causal effect of \emph{fixing} a set of features $S$, and therefore, yields Shapley values which are unaffected by correlations within the data. 
\newpage

\section{COMPUTING CONDITIONAL EXPECTATION SHAPLEY USING SUPERVISED APPROACH}\label{subsec:CES-comp}

In this work, when the conditional distribution is not tractable analytically, we use the supervised learning approach as in \citet{expondatamanifold} to estimate the conditional expectation in CES. We present this methodology in this section for completeness.

Here, we use the fact that the conditional expectation $g(\x_S) \coloneqq \E[f(\X) \mid \X_S = \x_S]$ minimises the mean squared error $\mathcal{L}_S(\tilde{g}) = \E_{\X \sim p(\X)}[f(\X) - \tilde{g}(\X_S)]^2$. Using this, we can define a surrogate model $g_\theta(\x_S)$ that takes as input coalition of features $\x_S$ (e.g., by masking features in $\bar{S}$) and that is trained to minimise the loss:
\begin{align*}
    \mathcal{L}(\theta) = \E_{\X \sim p(\X)}\E_{S \sim \textup{Shapley}}[f(\X) - g_\theta(\X_S)]^2
\end{align*}
Here, $S\sim \textup{Shapley}$ corresponds to sampling coalitions from the distribution where the probability assigned to each coalition is the combinatorial factor $|S|!(n - |S| - 1)!/n!$. As the surrogate model $g_\theta(\x_S)$ approaches the CES value function $\E[f(\X)\mid \X_S = \x_S]$, the loss $\mathcal{L}(\theta)$ is minimised.
\newpage

 \section{COMPUTING THE MANIFOLD $\man$}\label{subsec:computingman}
	\done{Add a reference to this in the main text}
		If we choose the manifold $\Z$ based on probability density/mass as outlined in Definitions \ref{den-manifold} and \ref{mass-manifold}, we must estimate the region $\man \coloneqq \{\textbf{x}: p(\textbf{x}) > \epsilon\}$. There are various ways to esimate this set, and the methodology used depends on dataset properties, such as the data dimensions, as well as the degree of accuracy sought. Below, we outline two such solutions, which can be used depending on the dataset dimensions:
	\subsection{Using Variational Auto-Encoders for manifold estimation}
	Variational Auto-Encoders (VAEs) have been a popular method of density estimation \citep{autoencoding-vb, vae-intro}.
	Instead of maximising the log likelihood, which may be intractable in general, VAE training involves maximising a lower bound of the log likelihood, called the Evidence Lower Bound (ELBO). In order to do so, the VAEs assume that data are generated from some random process, involving latent random variables $\textbf{z}$, and that a value $\x$ is sampled from a conditional $p_\theta(\textbf{x}\mid \textbf{z})$, also referred to as the likelihood \citep{vae-intro}.
	
	Let $q_\phi(\textbf{z} \mid \textbf{x})$ be a parametrized posterior. Then, we have that 
	\begin{align*}
		\log p_{\theta}(\textbf{x}) &= \log \int \frac{p_\theta(\textbf{x}\mid \textbf{z})p(\textbf{z})}{q_\phi(\textbf{z} \mid \textbf{x})} q_\phi(\textbf{z} \mid \textbf{x}) d\textbf{z} \geq \int \log \left(\frac{p_\theta(\textbf{x}\mid \textbf{z})p(\textbf{z})}{q_\phi(\textbf{z} \mid \textbf{x})}\right) q_\phi(\textbf{z} \mid \textbf{x}) d\textbf{z} \eqqcolon \textup{ELBO}_{\theta, \phi}(\textbf{x})
	\end{align*}
	
	The VAE training therefore involves maximising the expected lower bound $\E[\textup{ELBO}_{\theta, \phi}(\textbf{X})]$ over $\theta, \phi$. Let $(\theta^*, \phi^*) \coloneqq \arg\max_{(\theta, \phi)} \E[\textup{ELBO}_{\theta, \phi}(\textbf{X})]$, then we have that
	\[
	\log p_{\theta^*}(\textbf{x}) \geq \textup{ELBO}_{\theta^*, \phi^*}(\textbf{x}).
	\]
	Therefore, $\exp(\textup{ELBO}_{\theta^*, \phi^*}(\textbf{x})) \geq \epsilon$ implies that $p_{\theta^*}(\textbf{x}) \geq \epsilon$. We can use the ELBO to approximate the manifold:
	\[
	\man \approx \{\textbf{x}: \exp(\textup{ELBO}_{\theta^*, \phi^*}(\textbf{x})) \geq \epsilon\}.
	\]
	Assuming that $p_{\theta^*}(\textbf{x})$ is an accurate density model, we get that the above approximation of the $\epsilon$-manifold is going to be conservative in the sense that it will be a subset of the true $\man$. 
	\blue{Add: Use of kernel density estimators}
	
	\subsection{Thresholded Kernel Density Classification (tKDC)}
	Alternatively, we can use Kernel Density Estimation (KDE) to estimate the manifold $\man$. 
	KDE provides a way of estimating normalized density $\hat{p}(\x)$ using a finite dataset. This can be used to approximate most well-behaved smooth densities. Given $n$ datapoints, KDE provides an estimate $\hat{p}_n(\x)$ with mean squared error that shrinks at rate $O(n^{-\frac{4}{4+d}})$, where $d$ is the dimension of $\x$. This means that with enough data, KDE will identify an accurate density. The same may not be true for parametric methods of density estimation. However, evaluating the density $\hat{p}_n(\x)$ at a point $\x$ is prohibitively expensive when $n$ is large, as it involves the kernel contributions from every point in dataset. 
	
	To circumvent this problem, \citet{gan2017scalable} propose tKDC, a computationally efficient algorithm for classifying points with $p(\x) \geq \epsilon$ using KDE, where classification errors are allowed for densities within $\pm \epsilon\delta$ of the density threshold $\epsilon$ (for a given $\delta > 0$). 
	
	When the dimension of data, $d$, is large, the convergence of $\hat{p}_n(\x)$ is extremely slow. This is reflected in the error term $O(n^{-\frac{4}{4+d}})$ which shrinks slowly with increasing $n$ when $d$ is large. Therefore, while tKDC provides an asymptotically accurate density classification methodology, the convergence can be slow for large $d$, and in this case alternative methodologies like using VAEs may be more feasible.

	\subsection{Choosing the threshold $\epsilon$}
	Our manifold $\man$ is parameterized by a density threshold $\epsilon$. In practice, the probability density may depend on the dataset size, dimensionality and distribution, and as a result the range of density values may vary substantially among different datasets. It is therefore not possible to a priori define threshold values of $\epsilon$. Instead, we specify thresholds in terms of a probability mass $\alpha \in [0, 1]$. That is, we pick a fraction of the data that we choose to classify as having low density and set the threshold accordingly. This corresponds to picking the value of $\epsilon$ to be $\epsilon^{(\alpha)}$, where $\epsilon^{(\alpha)}\coloneqq \sup\{\epsilon\geq 0: \p(\X \in \man)\geq \alpha \}$. In practice, since we do not have access to the true density model, the estimation of $\epsilon^{(\alpha)}$ can be difficult, and we pick the threshold $\epsilon^{(\alpha)}$ based on the quantiles of the observed density estimates $\hat{p}(\x)$ for $\x$ in some held out dataset. The authors in \citet{density_threshold} show that for kernel density estimators this quantile converges to the ideal $\epsilon^{(a)}$ defined above.

\newpage
\section{EXPERIMENTAL RESULTS}\label{sec:exps-app}
\subsection{Experimental details for real-world dataset experiments}\label{subsec:experimental_detals_app}
Here, we explicitly define the models used for the real-world dataset experiments. 
\subsubsection{COMPAS dataset}
\paragraph{Biased classifier.} 
For COMPAS dataset, the biased classifier $f$ is defined as:
\[
f(\x) \coloneqq \ind(\x_{\textup{race}} = \textup{African American}).
\]
Here, $\x_{\textup{race}}$ denotes the race feature of the for the datapoint $\x$.

\paragraph{Out of manifold perturbation.} To perturb the model outside the manifold $\mathcal{Z}$, we construct 2 synthetic binary features (referred to as `unrelated columns') positively correlated with race. 

Let $Z_i$ be i.i.d. random variables with distribution $\textup{Bern}(0.90)$, then the `unrelated columns' feature correponding to the $i$'th datapoint, $\X^{(i)}_{\textup{unrelated column}}$, is defined as follows:
\[
\X^{(i)}_{\textup{unrelated column}} \coloneqq \ind(\X^{(i)}_{\textup{race}} = \textup{African American}) \times Z_i.
\]
Finally, the perturbed classifier model $g_\Z: \mathcal{X}\rightarrow \{0, 1\}$ is defined as follows:
\[
g_\Z(\x) \coloneqq \ind(\x \in \Z)\, f(\x) + \ind(\x \not\in \Z)\, \ind(\x_{\textup{unrelated column}} > 0).
\]

\subsubsection{Communities and crime dataset}
\paragraph{Biased classifier.} 
Likewise, for the CC dataset, the biased classifier $f$ is defined as:
\[
f(\x) \coloneqq \ind(\x_{\textup{percentage of caucasian population}} > 0).
\]

\paragraph{Out of manifold perturbation.}
We again construct 2 synthetic features (referred to as `unrelated columns').
Using the same random variables $Z_i$ as defined above, the `unrelated columns' feature corresponding to the $i$'th datapoint, $\X^{(i)}_{\textup{unrelated column}}$, is defined as follows:
\[
\X^{(i)}_{\textup{unrelated column}} \coloneqq \X^{(i)}_{\textup{percentage of caucasian population}} \times Z_i.
\]
Just like in COMPAS dataset experiments, the perturbed classifier model $g_\Z: \mathcal{X}\rightarrow \{0, 1\}$ is defined as follows:
\[
g_\Z(\x) \coloneqq \ind(\x \in \Z)\, f(\x) + \ind(\x \not\in \Z)\, \ind(\x_{\textup{unrelated column}} > 0).
\]

\subsection{Additional Experiments}
For all experiments in this section, we consider the causal structure in \citet{lshap} (see Figure \ref{fig:dag}), where the true features are formally distinguished from the input features. In this setting, the Interventional Shapley is equivalent to Marginal Shapley.\blue{we use KDE for RJBShap estimation}

	\subsubsection{Off-manifold perturbation}
	\begin{comment}
			\begin{figure}[t]
			\centering
			\begin{subfigure}[t]{0.46\textwidth}
				\centering
				\includegraphics[height=2.5in]{./images/perturbation_exp/classifier_heatmap_gt_model_gaussian_heatmap.png}
				\subcaption{The ground truth classifier}	
			\end{subfigure}
			\begin{subfigure}[t]{0.46\textwidth}
				\centering
				\includegraphics[height=2.5in]{./images/perturbation_exp/pert_3}
				\subcaption{Trained classifier $\hat{g}_\delta$ for $\delta = 3$}
				%
			\end{subfigure}\\
			\begin{subfigure}[t]{0.46\textwidth}
				\centering
				\includegraphics[height=2.5in]{./images/perturbation_exp/pert_6}
				\subcaption{Trained classifier $\hat{g}_\delta$ for $\delta = 6$}
				%
			\end{subfigure}
			\begin{subfigure}[t]{0.46\textwidth}
				\centering
				\includegraphics[height=2.5in]{./images/perturbation_exp/pert_9}
				\subcaption{Trained classifier $\hat{g}_\delta$ for $\delta = 9$}
				%
			\end{subfigure}
			\caption{Decision boundary for ground truth and adversarially perturbed classifiers. Each perturbed classifier has a test accuracy of at least 99\%.}\label{fig:pert1}
		\end{figure}
	\end{comment}

	\begin{figure*}[h!]
	\centering
	\begin{subfigure}{0.24\textwidth}
			\centering
			\includegraphics[height=1.2in]{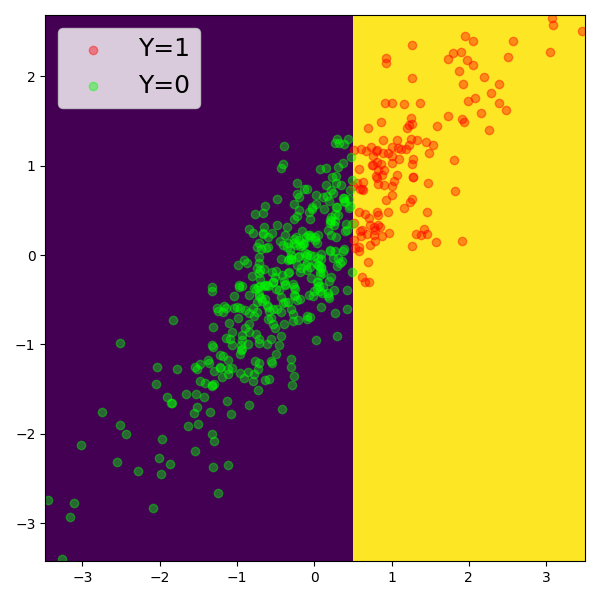}
			\subcaption{Classifier $g_\delta$ for $\delta = 0$.}
			\label{fig:pert-mans-0}
		\end{subfigure}%
	\begin{subfigure}{0.70\textwidth}
		\centering
		\includegraphics[height=1.2in]{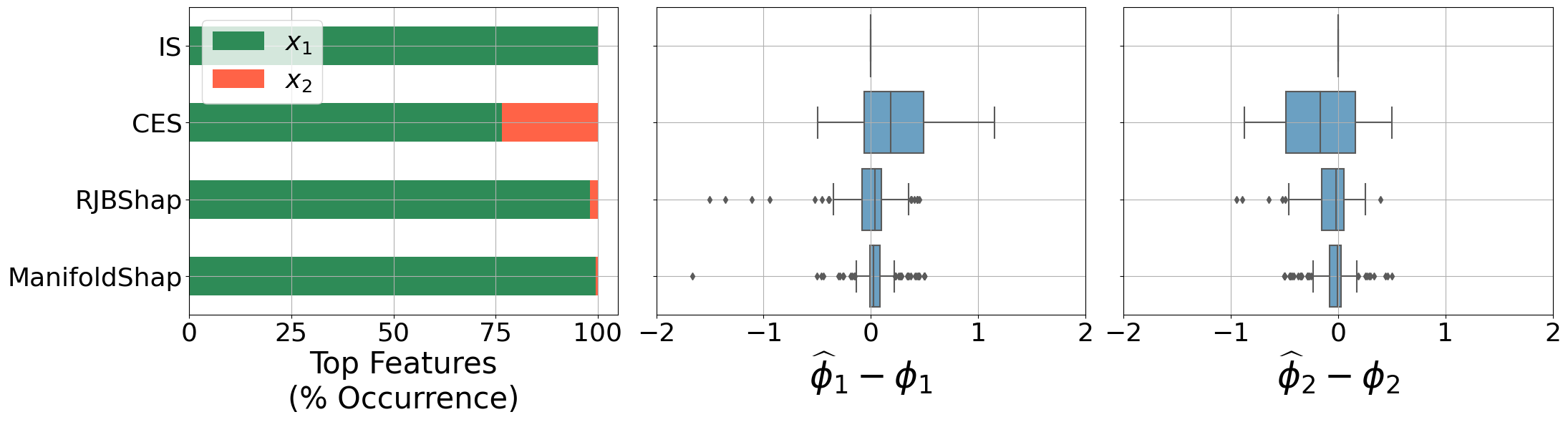}
		 \subcaption{$\delta = 0$}
		\label{fig:shaps0}
	\end{subfigure}\\
	\begin{subfigure}{0.24\textwidth}
			\centering
			\includegraphics[height=1.2in]{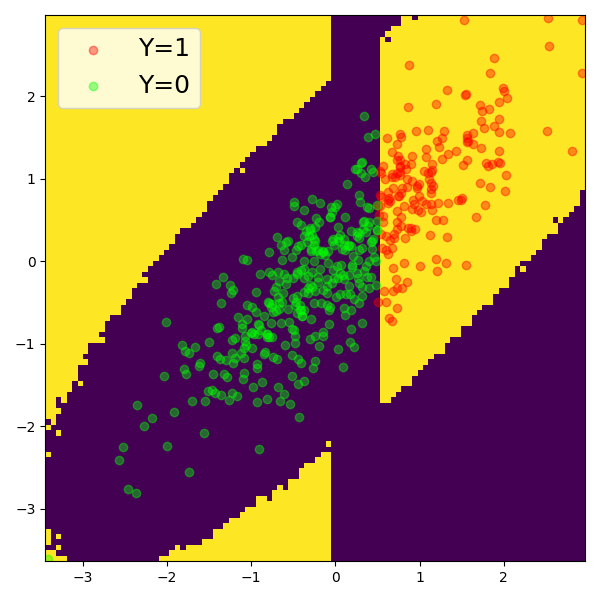}
			\subcaption{Classifier $g_\delta$ for $\delta = 10$.}
			 \label{fig:pert-mans-20}
		\end{subfigure}%
	\begin{subfigure}{0.70\textwidth}
		\centering
		\includegraphics[height=1.2in]{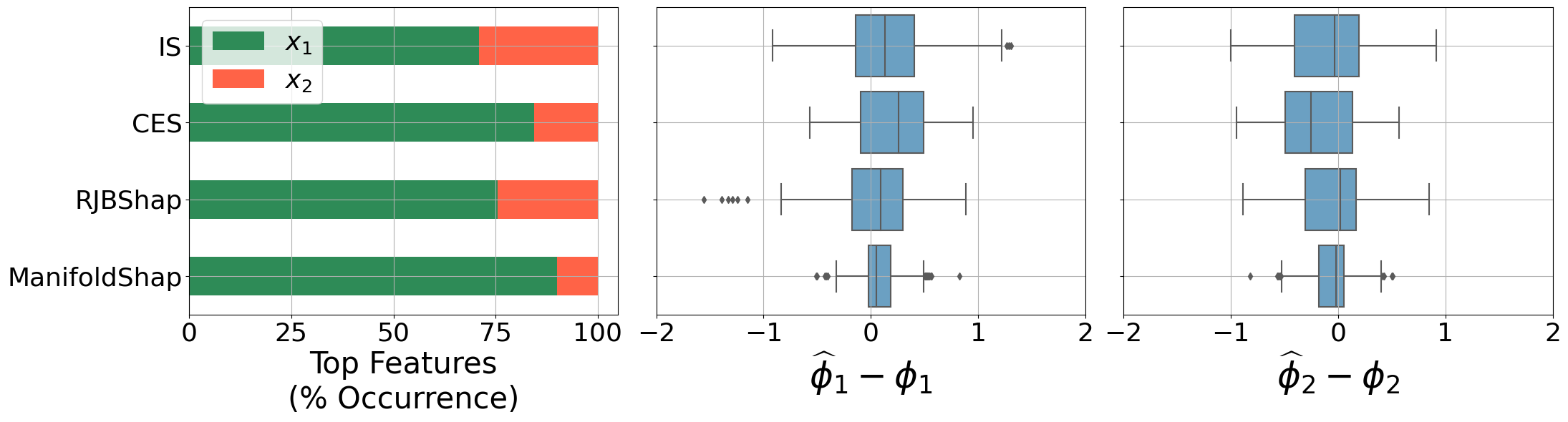}
		 \subcaption{$\delta = 10$}
		\label{fig:shaps10}
	\end{subfigure}
	\caption{\textbf{Left:} \ref{fig:pert-mans-0} and \ref{fig:pert-mans-20} show classifier decision boundaries for $\delta=0, 10$ respectively. \textbf{Right:} \ref{fig:shaps0} and \ref{fig:shaps10} show boxplots of $\hat{\phi}_i - \phi_i$ for $i \in \{1,2\}$ and different off-manifold perturbations.}\label{fig:pert2}
\end{figure*}

	In this experiment we investigate the effect of model perturbation in low density regions on Shapley values obtained using our methodology as well as other baselines. We do so by defining adversarial models, which agree with the ground truth model on the manifold $\alphaman$, but have been perturbed outside the manifold. 
	
	First, we define a ground truth data generating mechanism as described below. 

	\paragraph{Data generating mechanism.}
	
	In this experiment, $\mathcal{Y} = \{0, 1\}$ and $\mathcal{X} \subseteq \mathbb{R}^2$, where:
	\begin{align*}
		\X &\sim \mathcal{N}\left(\begin{pmatrix}
			0 \\
			0
		\end{pmatrix}, \begin{pmatrix}
			1 & 0.90 \\
			0.90 & 1
		\end{pmatrix}\right)\\
	 Y &\coloneqq \ind(X_1 > 1/2).
	\end{align*}

	Next, for the adversarial models, we define the following family of perturbed models.
	
	\paragraph{Perturbed models.}
	We define the following family of perturbed models $g_\delta:\mathcal{X} \rightarrow \{0,1\}$, parameterised by $\delta \in \mathbb{R}$.
	\begin{align*}
		g_\delta(\X) \coloneqq Y \ind(\X \in \alphaman) + \ind((1-\delta) X_1 > 1/2) \ind(\X \not \in \alphaman).
	\end{align*}
	Here, we use VAEs to estimate $\alphaman$ as described in Section \ref{subsec:computingman}, and choose $\alpha=1-10^{-3}$.

	By construction, the classifiers $g_\delta$ should agree with the ground truth on the $\alpha$-manifold, i.e. $g_\delta(\X) = Y$ when $\X \in \alphaman$. However, these classifiers differ from the ground truth outside the $\alpha$-manifold. Figures \ref{fig:pert-mans-0} and \ref{fig:pert-mans-20} show the classifier decision boundaries, along with the original data $(X_1, X_2)$. Each of these classifiers have a test accuracy of at least 99.5\%, and therefore it is impossible to distinguish between them on the data manifold. However, as we will show next, the Interventional Shapley values computed for these classifiers are drastically different.

\begin{figure*}[t]
	\centering
	\includegraphics[height=2.5in]{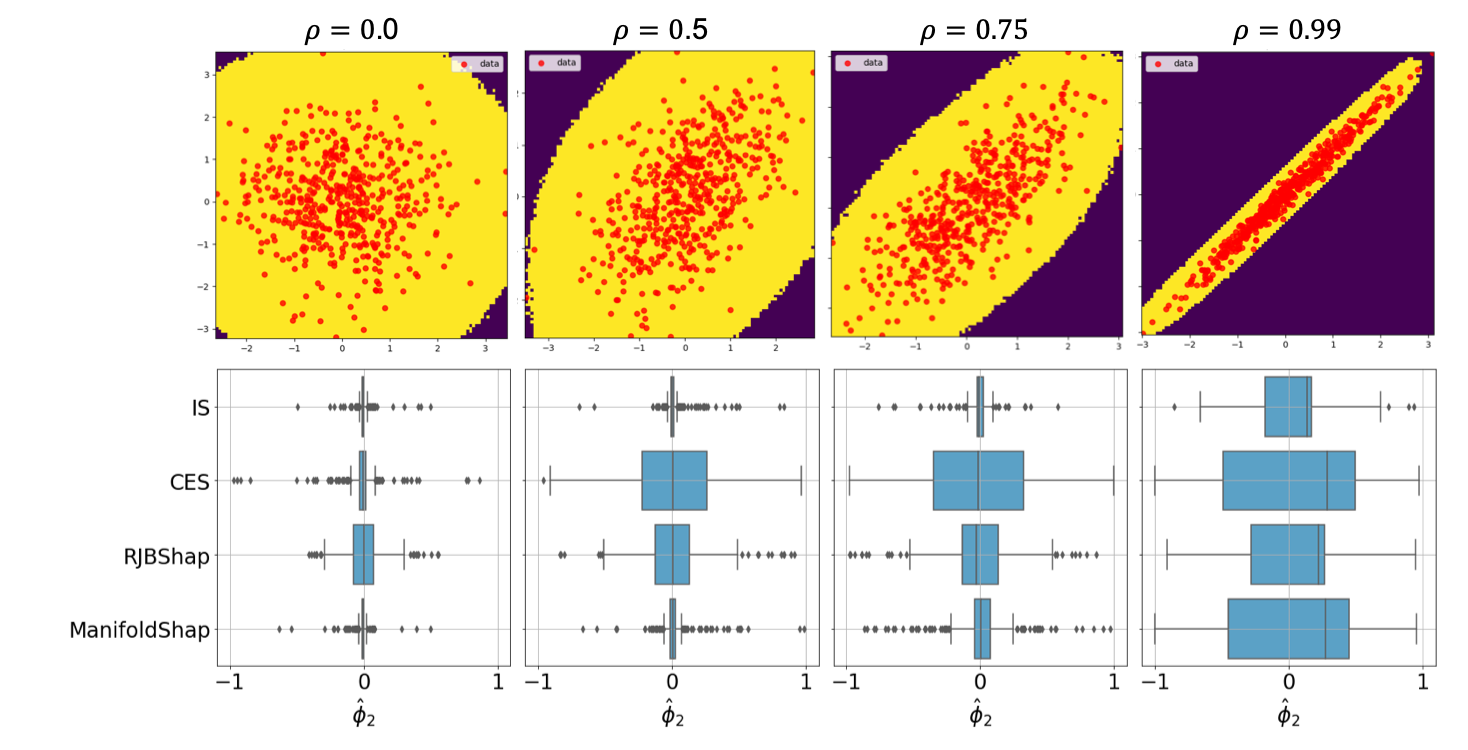}
	\caption{\textbf{Fig (a) (above):} Manifolds $\man$ for different values of $\rho$. The manifold $\man$ is denoted by the yellow region. \textbf{Fig (b) (below):} Boxplots of $\hat{\phi}_2$ for increasing values of $\rho$.}\label{fig:corr2}
    \end{figure*}

\paragraph{Estimating conditional expectation for CES.} 
	For the data generating mechanism described above, the conditional distributions $p(X_2 \mid X_1)$ and $p(X_1 \mid X_2)$ are tractable. In fact, it is straightforward to get that $X_2\mid X_1 \sim \mathcal{N}( 0.90 * X_1, 1-0.90^2)$, and similarly for $X_1\mid X_2$. We use this to estimate the conditional expectation using $m$ Monte Carlo samples from the conditional distributions:
	\begin{align*}
	    \E[f(\X)\mid X_1=x_1] &\approx \frac{1}{m} \sum_{i=1}^m f(x_1, X^i_2)  \hspace{0.3cm} \textup{where }\\
	    X^i_2 &\overset{\textup{i.i.d.}}{\sim} \mathcal{N}( 0.90 * x_1, 1-0.90^2), 
	\end{align*}
	and similarly for $\E[f(\X)\mid X_2=x_2]$. In this experiment, we use $m = 500$. \blue{double-check mathematically correct}

\paragraph{Results.}
We compute the Shapley values for the models on 500 datapoints from a held-out dataset.
In Figure \ref{fig:pert2} we plot the difference between estimated Shapley values and the ground truth Interventional Shapley values, for different methodologies. For a fair comparison between different value functions, we normalise the Shapley values so that $\sum_{i\in \{1,2\}} |\phi_i | = 1$. 

To compute the ground truth IS values $\{\phi_i\}_{i\in \{1,2\}}$, we use the ground truth function $f(\X) = \ind(X_1 > 1/2)$ instead of the perturbed model $g_\delta$, and therefore, the ground truth IS values do not change with increasing off-manifold perturbations $\delta$. Moreover, note that since the ground truth model $f$ is independent of $X_2$, the Shapley values for feature $2$, $\phi_2 = 0$.

The figure shows that when $\delta=0$, i.e., there is no off-manifold perturbation, the errors for IS values, i.e., $\hat{\phi}_i - \phi_i$, are 0. This is because the IS values are equal to the ground truth in this case, as the perturbed model $g_\delta$ is equal to the ground truth model $f$ everywhere. 
CES, on the other hand gives biased Shapley values, as can be seen from the errors $\hat{\phi}_i - \phi_i$ being concentrated away from 0. This happens because of the high positive correlation between the features -- conditional expectation is highly sensitive to feature correlations, unlike marginal expectation. 

For ManifoldShap, the errors are less concentrated around 0 than for IS values. This highlights the reduced causal accuracy in ManifoldShap as a result of restricting function evaluations to the manifold $\alphaman$. However, ManifoldShap values are more accurate than the CES and RJBShap values.

 It can be seen that for $\delta = 10$, the errors in IS values are highest among all the baselines. This highlights the off-manifold nature of IS values, i.e., perturbing model in low-density regions can significantly change the computed Shapley values. Additionally, CES values are biased as the errors are concentrated away from 0. ManifoldShap errors remain largely restricted between -0.2 and 0.2, with error distribution concentrated around 0.

The barplots in Figures \ref{fig:shaps0} and \ref{fig:shaps10} show the most important features as per different value functions for $\delta=0, 10$ respectively. The figure shows that for ground truth model ($\delta=0$), IS values attribute the greatest importance to feature 1 for all datapoints. This is expected as the ground truth model does not depend on $x_2$. For CES, on the other hand,  feature 2 receives greater importance for roughly 25\% datapoints. This is again due to the positive correlation between the features $X_1, X_2$. For $\delta=10$, Figure \ref{fig:shaps10} shows that ManifoldShap attributes least importance to feature 2, among all baselines considered.

The results show that CES and ManifoldShap are in practice less sensitive to off-manifold manipulation compared to IS and RJBShap, however, CES values may be biased when features are highly correlated, whereas ManifoldShap remain closer to the ground truth IS values overall.

\subsubsection{Sensitivity to correlations}\label{subsec:corr}

\begin{comment}

	\begin{figure}[ht]
	\centering
	\begin{subfigure}[t]{0.46\textwidth}
		\centering
		\includegraphics[height=2.5in]{./images/increasing_correlations/man_0.png}
		%
	\end{subfigure}
	\begin{subfigure}[t]{0.46\textwidth}
		\centering
		\includegraphics[height=2.5in]{./images/increasing_correlations/man_0.33.png}
		%
		%
	\end{subfigure}\\
	\begin{subfigure}[t]{0.46\textwidth}
		\centering
		\includegraphics[height=2.5in]{./images/increasing_correlations/man_0.66.png}
		%
		%
	\end{subfigure}\
	\begin{subfigure}[t]{0.46\textwidth}
		\centering
		\includegraphics[height=2.5in]{./images/increasing_correlations/man_0.99.png}
		%
%
	\end{subfigure}
	\caption{Manifolds $\man$ for different values of $\rho$.}\label{fig:corr2}
	%
\end{figure}
\end{comment}

	In this experiment we investigate the sensitivity of ManifoldShap to increasing correlation among the features, as compared to the other baselines. To this end, we define the following family of data distributions:
	
	\paragraph{Data generating mechanism.}
	In this experiment, $\mathcal{X} \subseteq \mathbb{R}^2$ and $\mathcal{Y} \subseteq \mathbb{R}$. Specifically, 
	\begin{align*}
		\X &\sim \mathcal{N}\left(\begin{pmatrix}
			0 \\
			0
		\end{pmatrix}, \begin{pmatrix}
			1 & \rho \\
			\rho & 1
		\end{pmatrix}\right)  \hspace{0.3cm} \textup{where, $\rho \in (-1, 1)$.}
	\end{align*}
	Moreoever, the ground truth model under consideration is $Y \coloneqq f(\X) = X_1$.
	The parameter $\rho$ corresponds to the correlation between $X_1$ and $X_2$. When $\rho = 0$, $X_1$ and $X_2$ are independent random variables. As $\rho$ increases (decreases), the features get more positively (negatively) correlated. Figure \ref{fig:corr2}a shows the data generated for different values of $\rho$, along with the $\epsilon$-manifold. Here, we choose $\epsilon$ to be the 1st percentile of density values on a held out dataset, i.e. $\epsilon \approx \epsilon^{(0.99)}$.
	
	Using the data generating mechanism described above, we generate data for a given $\rho$, which is then used to estimate Shapley values.
	Note that since $f$ is independent of $X_2$, we would ideally expect the Shapley values corresponding to feature $X_2$, i.e., $\hat{\phi}_2$, to be close to 0. 
	
	\paragraph{Estimating conditional expectation for CES.} 
	Like the previous experiment, the conditional distributions $p(X_2 \mid X_1)$ and $p(X_1 \mid X_2)$ are tractable for the data generating mechanism described above. In fact, $X_2\mid X_1$ follows the Gaussian distributions $\mathcal{N}( \rho X_1, 1-\rho^2)$ (and similarly for $X_1\mid X_2$). We use this to estimate the conditional expectation using $m=500$ Monte Carlo samples like described for previous experiment.

	\begin{figure*}[t]
    \centering
    \includegraphics[height=2.5in]{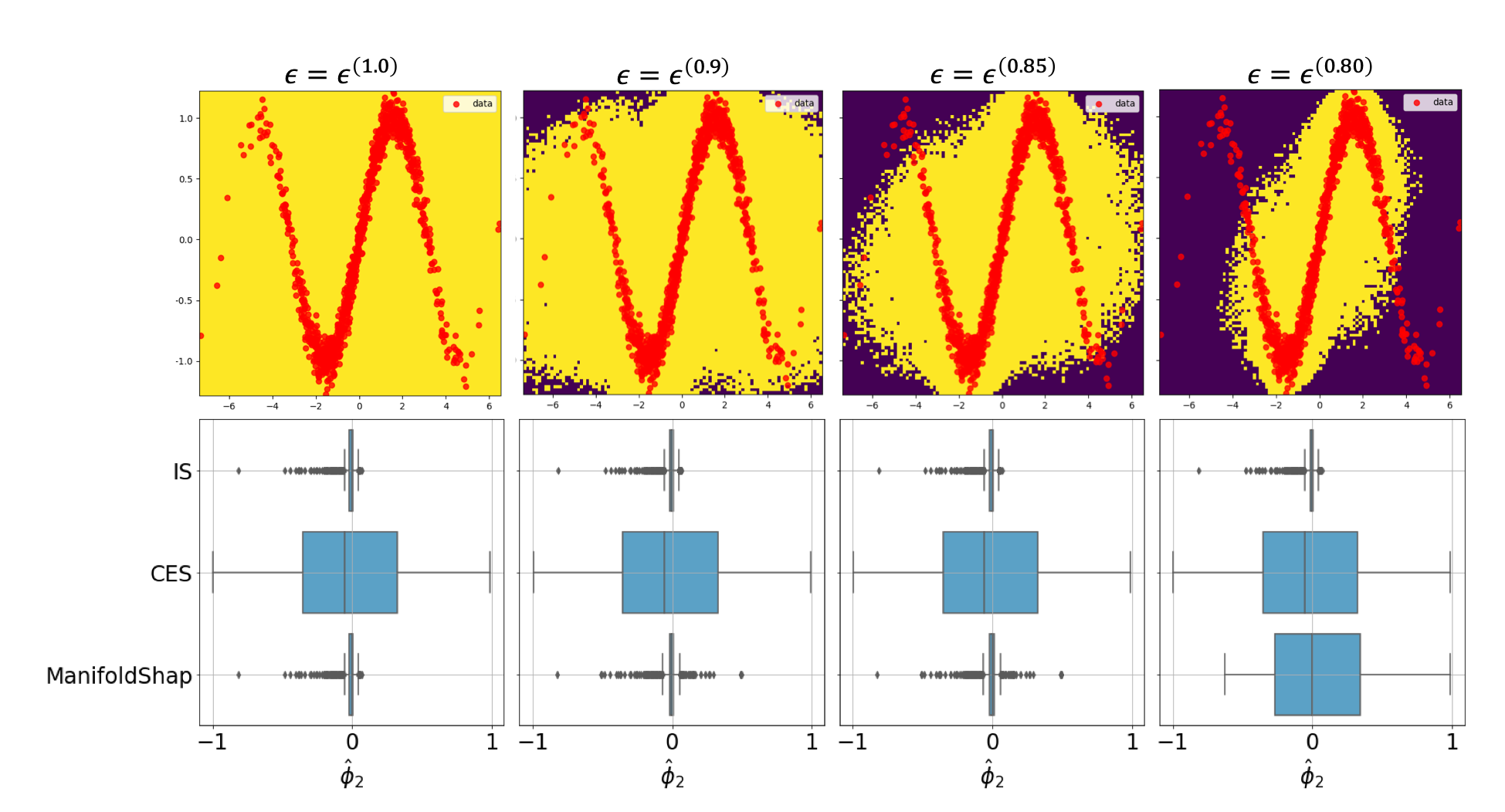}
    \caption{\textbf{Fig (a) (above):} Manifolds $\man$ for different values of $\epsilon$. The manifold $\man$ is denoted by the yellow region. \textbf{Fig (b) (below):} Boxplots of $\hat{\phi}_2$ for increasing values of $\epsilon$.}\label{fig:inc_eps1}
    \end{figure*}
	
	\paragraph{Results.}
	Figure \ref{fig:corr2}b shows the boxplots of $\hat{\phi}_2$ values for the different methodologies and values of correlation $\rho$. For $\rho = 0$, the features are independent, and in this case the CES and IS values are expected to be equal, as the conditional expectation is same as marginal expectation in this case. Therefore, we observe in Figure \ref{fig:corr2}b that CES values are close to IS values when $\rho=0$.

	As $\rho$ increases, the distribution of CES values $\hat{\phi}_2$ gets more spread out, away from the ground truth value of 0. In comparison, both ManifoldShap and IS values remain concentrated around 0, with IS remaining closer to 0. This happens because IS values are not sensitive to feature correlations. Furthermore, ManifoldShap values are significantly less sensitive to increasing $\rho$ as compared to CES values.
	In comparison, it can be seen that the distribution of RJBShap values gets wider as $\rho$ increases, showing that RJBShap is more sensitive to increasing correlation than ManifoldShap values. This is because RJBShap values explicitly depend on the density values $p(\x)$ which changes with changing values of $\rho$.
	
	Finally, when $\rho=0.99$, the features are highly correlated. In this case, the manifold $\man$ is sparse, and as a result the ManifoldShap and CES behave similarly. This is evident from the fact that the boxplots of $\hat{\phi}_2$ for ManifoldShap and CES in Figure \ref{fig:corr2}b are very similar when $\rho=0.99$. This also highlights a potential failure mode of ManifoldShap: when the manifold $\Z$ is sparse, the ManifoldShap may behave similarly to CES, leading to unintuitive explanations.
	
%

%
%
%
%
%
%
%
%
%
%
%
%
%
%
%
%
%
%
%
%
%
%
%
%
%
%
%
%
%

\begin{comment}
\begin{figure}
	\centering
	\begin{subfigure}[t]{0.46\textwidth}
		\centering
		\includegraphics[height=2.3in]{./images/increasing_correlations2/rho0.png}
		 \subcaption{$\rho=0$}
		 \label{fig:corr0}
	\end{subfigure}
	\begin{subfigure}[t]{0.46\textwidth}
		\centering
		\includegraphics[height=2.3in]{./images/increasing_correlations2/rho0.33.png}
		\subcaption{$\rho=0.33$}
		%
	\end{subfigure}\\
	\begin{subfigure}[t]{0.46\textwidth}
		\centering
		\includegraphics[height=2.3in]{./images/increasing_correlations2/rho0.66.png}
		\subcaption{$\rho=0.66$}
		%
	\end{subfigure}
	\begin{subfigure}[t]{0.46\textwidth}
		\centering
		\includegraphics[height=2.3in]{./images/increasing_correlations2/rho0.99.png}
		\subcaption{$\rho=0.99$}
		%
	\end{subfigure}
	\caption{Boxplots of $\hat{\phi}_2$ for different values of $\rho$.}\label{fig:corr}
	%
\end{figure}
\end{comment}

\subsubsection{Dependence on manifold size}
\begin{comment}
	\begin{figure}
	\centering
	\begin{subfigure}[t]{0.22\textwidth}
		\centering
		\includegraphics[height=1.5in]{./images/sine_wave2/manifolds_against_eps=0.0_sine_wave_x_1.png}
		%
		%
	\end{subfigure}
	\begin{subfigure}[t]{0.22\textwidth}
		\centering
		\includegraphics[height=1.5in]{./images/sine_wave2/manifolds_against_eps=0.00048_sine_wave_x_1.png}
		%
		%
	\end{subfigure}
	\begin{subfigure}[t]{0.22\textwidth}
		\centering
		\includegraphics[height=1.5in]{./images/sine_wave2/manifolds_against_eps=0.0039_sine_wave_x_1.png}
		%
		%
	\end{subfigure}
	\begin{subfigure}[t]{0.22\textwidth}
		\centering
		\includegraphics[height=1.5in]{./images/sine_wave2/manifolds_against_eps=0.032_sine_wave_x_1.png}
		%
		%
	\end{subfigure}
	\caption{Manifolds $\man$ for different values of $\epsilon$.}\label{fig:inc_eps}
	%
\end{figure}
\end{comment}
    In this experiment, we investigate how the ManifoldShap values change as the size of $\Z$ decreases. In particular, we investigate the relationship between ManifoldShap, IS and CES as the manifold $\Z=\man$ gets smaller. To do so, we consider $\epsilon \in \{\epsilon^{(1.0)}, \epsilon^{(0.9)}, \epsilon^{(0.85)}, \epsilon^{(0.80)}\}$, where $\epsilon^{(\alpha)}$ is as defined in definition \ref{mass-manifold}. We carry out this experiment on the following data generating mechanism, with $\mathcal{X} \subseteq \mathbb{R}^2$ and $\mathcal{Y}\subseteq \mathbb{R}$:
    
    \paragraph{Sine Wave.}

    \begin{align*}
        X_1 \sim \mathcal{N}(0, 4); \hspace{0.4cm} X_2 \mid X_1 \sim \mathcal{N}(\sin(X_1), 0.01).
    \end{align*}

    Moreoever, the ground truth model under consideration is $Y \coloneqq f(\X) = X_1$.
    Using the data generating mechanism described above, we generate data which is then used to compute Shapley values of $f$.
    Figure \ref{fig:inc_eps1}a shows how the $\epsilon$-manifolds shrinks as $\epsilon$ increases from 0. 
    Here, we use the supervised approach described in Section \ref{subsec:CES-comp} to compute conditional expectation for CES values since the conditional $X_1\mid X_2$ is not easily tractable.
    
    \paragraph{Results for Sine Wave.}

    Figure \ref{fig:inc_eps1} shows the boxplots of $\hat{\phi}_2$ for different values of $\epsilon$. Recall that since the ground truth function is independent of $X_2$, the ground truth IS value $\phi_2=0$. When $\epsilon = \epsilon^{(1.0)}= 0$, ManifoldShap is equivalent to IS, and therefore the values in Figure \ref{fig:inc_eps1} are identical. CES values, on the other hand are concentrated away from 0. This happens because the features are highly coupled in this experiment. 
    
    As $\epsilon$ increases, the ManifoldShap values for $\hat{\phi}_2$ get increasingly spread out. This shows that increasing $\epsilon$ may reduce the causal accuracy of computed ManifoldShap values, despite making them more robust to off-manifold perturbations. However, it is important to note that relative to CES, ManifoldShap values are closer to the ground truth than Interventional Shapley values for $\epsilon\in\{\epsilon^{(1.0)}, \epsilon^{(0.9)}, \epsilon^{(0.85)}\}$. When $\epsilon = \epsilon^{(0.80)}$, the ManifoldShap values for $\hat{\phi}_2$ are no longer concentrated around the ground truth value 0, as the manifold $\man$ excludes a significant number of input samples. This shows that ManifoldShap values may become inaccurate when the set $\Z$ becomes `small' relative to the true data manifold.
    
    \subsubsection{ManifoldShap vs RJBShap}\label{subsec:rjbshap}
    \begin{figure}[t]
        \centering
        \includegraphics[height=1.5in]{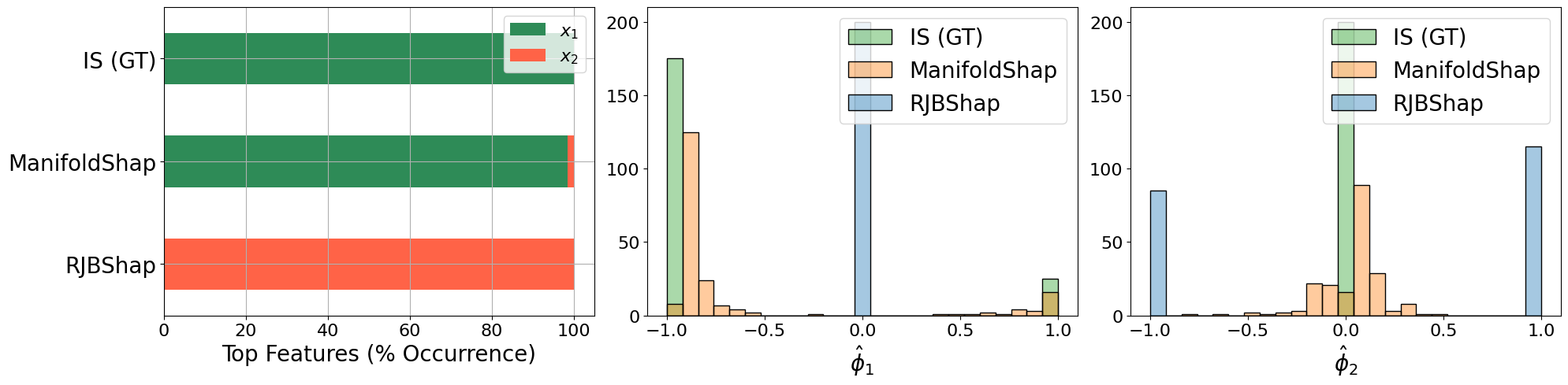}
        \caption{\textbf{Left (a):} Top features according to different Shapley value functions. \textbf{Right (b):} Histograms of computed Shapley values using different value functions. `IS (GT)' refers to the ground truth Interventional Shapley values.}
        \label{fig:rjbshap}
    \end{figure}
    In this experiment, we demonstrate that RJBShap provides explanations for  $\tilde{f}_p(\x)\coloneqq f(\x)p(\x)$ which is fundamentally different from $f(\x)$. The explanations obtained can therefore be misleading.
    
    \paragraph{Data generating mechanism.}
    In this experiment, $\mathcal{X} \subseteq \mathbb{R}^2$ and $\mathcal{Y} \subseteq \mathbb{R}$. Specifically,
    \begin{align*}
        \X &\sim \mathcal{N}\left(\begin{pmatrix}
			0 \\
			0
		\end{pmatrix}, \begin{pmatrix}
			1 & 0 \\
			0 & 1
		\end{pmatrix}\right)
    \end{align*}
    Moreover, the ground truth model under consideration is $Y\coloneqq f(\X) = \exp(X^2_1/2)$. The model is deliberately chosen so that $f(\X)p(\X)$ only depends on $X_2$, while $f(\X)$ only depends on $X_1$. Therefore, the Interventional Shapley value for feature 2 is 0, i.e., $\phi_2 = 0$. In contrast, the RJBShap value for feature 1 is 0, since $\tilde{f}_p(\X)$ is independent of feature 1. Therefore, if we use RJBShap to explain the function $f(\X)$, we would be misled into concluding that the function is independent of feature 1, when in reality the function is independent of feature 2.
    
    For this experiment, we use $\Z=\alphaman$ with $\alpha=1-10^{-3}$ to compute ManifoldShap, and use the ground truth data density to compute RJBShap values.
    \paragraph{Results.}
    Using the data-generating mechanism defined above, we generate datapoints $\{\x^{(i)} \}_{i=1}^{500}$, and compute Shapley values for these datapoints.
    Figure \ref{fig:rjbshap}a shows the top features according to different value functions. The results confirm that, for RJBShap the top feature is feature 2 for all datapoints, whereas for ground truth Interventional Shapley value feature 1 is most important for all datapoints. ManifoldShap values remain significantly closer to ground truth IS values, as over 98\% datapoints have feature 1 as the most important feature. 
    
    Figure \ref{fig:rjbshap}b shows the histograms of Shapley values for different value functions. For a fair comparison, we normalise the Shapley values so that $\sum_{i\in\{1, 2\}} |\phi_i| = 1$. It can be seen that ManifoldShap values are very close to the ground truth IS values, while the RJBShap values provide a stark contrast to the ground truth IS values. For IS values, $|\phi_1| = 1$ and $\phi_2 = 0$ which accurately reflects the fact that the function $f$ only depends on feature 1. In contrast, for RJBShap, $\phi_1 = 0$ and $|\phi_2| = 1$.
    
    \subsubsection{Accuracy with increasing feature space dimensions}\label{subsec:feature_dims}
    In this experiment, we illustrate how the accuracy of computed Shapley values varies with increasing dimensions of the feature space. Here, we consider dimensions of feature space $d\in \{100, 200, 500\}$.
    \paragraph{Data generating mechanism.}
    Here, $\X \in \mathbb{R}^d$ and $\mathcal{Y} \subseteq \mathbb{R}$. Specifically,
    \begin{align*}
        \X  &\sim\mathcal{N}(\textbf{0}_d, \Sigma^d), \quad \textup{where} \quad \Sigma^d_{ij} = \ind(i=j) +  0.9\ind(i\neq j)\\
        \quad Y &= X_1.
    \end{align*}
    In this example, the correlation between any two features is 0.90. This high positive correlation among features restricts support size of the data. 
    Additionally, the model under consideration is the perturbed model
    \begin{align*}
        f(\X) = Y + 10\,X_2 \ind(\X \not\in \alphaman).
    \end{align*}
    We use VAEs to estimate $\alphaman$ as described in Section \ref{subsec:computingman}, and choose $\alpha=1-10^{-3}$.

    Like in previous experiments, the model only depends on the first feature $X_1$ 
    on the $\alpha$-manifold, i.e., $f(\X) = X_1$ when $\X \in \alphaman$. Since $\alphaman$ contains $99.9\%$ of the data, the mean squared error of $f(\X)$ is very small (of the order $O(10^{-3})$).

    \paragraph{Results.}
    \begin{figure}[t]
        \centering
        \includegraphics[height=1.5in]{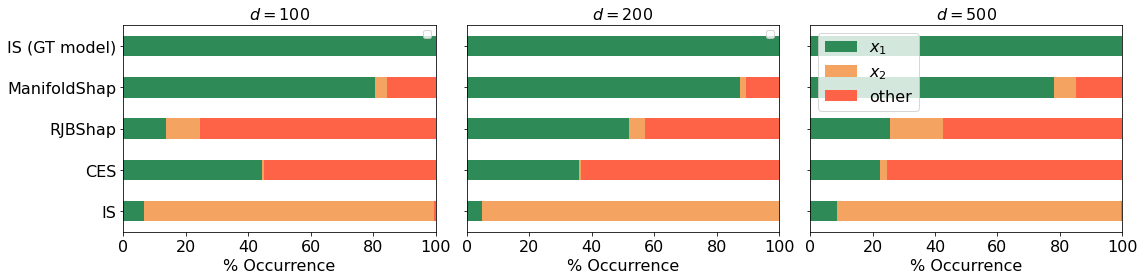}
        \caption{Top features according to different Shapley value functions for increasing feature dimensions $d$. `IS (GT)' refers to the ground truth Interventional Shapley values.}
        \label{fig:dimension_exps}
    \end{figure}
    Using the data generating mechanism described above, we generate datapoints $\{\x^{(i)} \}_{i=1}^{500}$, and compute Shapley values for these datapoints. In this example, we use the supervised approach in Section \ref{subsec:CES-comp} to compute the conditional expectation for CES values, and we use the rejection sampling procedure in Section \ref{subsec:rejection_sampling} to estimate ManifoldShap values.  

    Figure \ref{fig:dimension_exps} shows the top features according to different value functions for $d \in \{100, 200, 500\}$. The results show that, IS values attribute greater importance to feature 2 for more than 90\% of the datapoints, for all values of $d$ under consideration, while the remaining datapoints have feature 1 as the most important feature. 
    Similarly, both RJBShap and CES attribute greatest importance to feature 1 for less than 50\% of the datapoints. Moreover, the importance that CES attributes to feature 1 decreases with increasing dimensions. Intuitively, this happens because the number of features highly correlated with feature 1 increases as $d$ increases, leading CES to divide the attributions among increasing number of features. 
    Among all the baselines considered, ManifoldShap remains closest to the ground truth Shapley values as it attributes greatest importance to feature 1 for more than 80\% of the datapoints even as $d$ increases. 
    \subsubsection{Sensitivity of ManifoldShap and RJBShap to density estimation errors}\label{subsec:sensitivity-density-error}
    In this experiment we investigate the sensitivity of computed Shapley values using ManifoldShap and RJBShap, to increasing density estimation errors. 
    Here, $\X \in \mathbb{R}^{50}$ and $\mathcal{Y} \subseteq \mathbb{R}$. Specifically, we use the same data generating mechanism and model used in Section \ref{subsec:feature_dims} with $d=50$. 
    We use VAEs to obtain a density estimate $\hat{p}(\x)$, which is subsequently used to estimate RJBShap and $\alphaman$ for ManifoldShap estimation.  
    We generate datapoints $\{\x^{(i)}\}_{i=1}^{500}$, and compute ManifoldShap and RJBShap values for these datapoints, using density estimates of differing quality obtained by training VAE for different number of epochs. 

    Table \ref{tab:vae_iters} shows the percentage of datapoints with feature 1 as the most important feature as per each value function, for different density estimates. We use the oracle density of $\X$ to estimate density mean squared error in table \ref{tab:vae_iters}. It can be seen that ManifoldShap is significantly less sensitive to density estimation errors as compared to RJBShap. This is because ManifoldShap only depends on the density estimate via the indicator $\ind(\hat{p}(\x)\geq \epsilon^{(\alpha)})$, whereas RJBShap depends on the density explicitly. 

    \begin{table}
      \centering
        \caption{The percentage of datapoints with feature 1 as the most important feature as per each value function, for different density estimates obtained by training VAE for different number of epochs.}
    \label{tab:vae_iters}
        \begin{tabular}[width=0\textwidth]{lllll}
        \toprule
        No. of epochs & 10 & 50 & 200 & Oracle density \\
        \midrule
        ManifoldShap           &                    \textbf{79.9}&                    \textbf{78.5}& \textbf{80.2} & \textbf{81.0}\\
        RJBShap &          17.5 &         15.0 &        10.0 & 13.1\\
        \midrule
        Density MSE & 395.1 & 220.0 & 63.7 & 0.0 \\
        \bottomrule
        \end{tabular}
    \end{table}

\end{document}